\documentclass{article}

\usepackage{arxiv}
\usepackage[utf8]{inputenc} %BIBLATEX NEEDS THIS
\usepackage[english]{babel} % BIBLATEX NEEDS THIS
\usepackage[utf8]{inputenc} % allow utf-8 input
\usepackage[T1]{fontenc}    % use 8-bit T1 fonts
\usepackage{hyperref}       % hyperlinks
\usepackage{url}            % simple URL typesetting
\usepackage{booktabs}       % professional-quality tables
\usepackage{amsfonts}       % blackboard math symbols
\usepackage{nicefrac}       % compact symbols for 1/2, etc.
\usepackage{microtype}      % microtypography
\usepackage{lipsum}
\usepackage{caption}
\usepackage{subcaption}
\usepackage{wrapfig}
\usepackage{array}
\usepackage{tabularx}
\usepackage{booktabs}
\usepackage{amssymb}
\usepackage{amsmath}
\usepackage{mathtools}
\usepackage{amsthm}
\usepackage[
  separate-uncertainty = true,
  multi-part-units = repeat
]{siunitx}

\newtheorem{theorem}{Theorem}
\newtheorem{proposition}{Proposition}
\newtheorem{lemma}{Lemma}

\theoremstyle{definition}
\newtheorem{definition}{Definition}

\theoremstyle{remark}

\newcommand{\norm}[1]{\left\lVert#1\right\rVert}
\newcommand{\mynorm}[1]{\lVert#1\rVert}

\newcommand\inner[2]{\langle #1, #2 \rangle}

\newcommand{\tr}{\text{trace}}

\newcommand{\gamMax}{\gamma_{\max}}
\newcommand{\gamMin}{\gamma_{\min}}
\newcommand{\myrounds}[1]{\left ( #1 \right )}
\newcommand{\Ktilde}{\Tilde{K}}

\newcommand{\mybraces}[1]{\left \{ #1 \right \}}
\newcommand{\mysquares}[1]{\left [ #1 \right ]}

\newcommand{\slim}{\sum \limits}

\DeclarePairedDelimiter\abs{\lvert}{\rvert}%

\newcommand{\E}{\mathbb{E}}

\newcommand{\psione}{\psi_{1}}
\newcommand{\psitwo}{\psi_{2}}

\usepackage{csquotes}
\usepackage[style=authoryear,sorting=nty,natbib,maxbibnames=99]{biblatex}

\title{On the optimality of kernels for high-dimensional clustering}

\author{
  Leena Chennuru Vankadara \\
  IMPRS-IS\thanks{International Max Planck Research School for Intelligent Systems} , University of T\"{u}bingen\\
   \And
 Debarghya Ghoshdastidar \\
  Technical University of Munich \\
}

\begin{document}
\maketitle

\begin{abstract}
This paper studies the optimality of kernel methods in high-dimensional data clustering. Recent works have studied the large sample performance of kernel clustering in the high-dimensional regime, where Euclidean distance becomes less informative. However, it is unknown whether popular methods, such as kernel k-means, are optimal in this regime. We consider the problem of high-dimensional Gaussian clustering and show that, with the exponential kernel function, the sufficient conditions for partial recovery of clusters using the NP-hard kernel k-means objective matches the known information-theoretic limit up to a factor of $\sqrt{2}$ for large $k$. It also exactly matches the known upper bounds for the non-kernel setting. We also show that a semi-definite relaxation of the kernel k-means procedure matches upto constant factors, the spectral threshold, below which no polynomial time algorithm is known to succeed. This is the first work that provides such optimality guarantees for the kernel k-means as well as its convex relaxation. Our proofs demonstrate the utility of the less known polynomial concentration results for random variables with exponentially decaying tails in higher-order analysis of kernel methods.
\end{abstract}

% keywords can be removed

\section{Introduction}

Kernel methods are one of the most empirically successful class of machine learning techniques. While being easy to implement, kernel methods are well known to improve empirical performance of algorithms and are also related to other successful machine learning principles such as Gaussian process and neural networks \parencite{kanagawa2018gaussian,jacot2018neural}. At the heart of kernel based learning lies the \emph{kernel trick} which implicitly maps the data to a high, possibly infinite, dimensional \emph{reproducing kernel Hilbert space} (RKHS), and hence, induces non-linearity into classical linear learning models such as support vector machines, principle component analysis or k-means. Kernel methods are based on a solid theoretical foundation, which makes them conducive to theoretical analysis. There has been considerable theoretical research on kernel based supervised learning from a statistical perspective \parencite{steinwart2008support,mendelson2010regularization}, and to some extent, in the context of semi-supervised learning \parencite{wasserman2008statistical,mai2018random}. Perhaps surprisingly, much less is known about the statistical performance of kernel methods beyond such settings, for instance, kernel based data clustering.

A long-standing issue in the theoretical study of clustering, and also kernel based clustering, has been the lack of a universally accepted notion of \emph{goodness} of clustering. A popular definition of good clustering is one that consistently or near-optimally partitions the data domain. Based on this perspective, there exist \emph{approximation guarantees} for solving kernel based cost functions \parencite{wang2019scalable} and \emph{consistency results} showing that the clustering asymptotically approaches a limiting clustering \parencite{vonluxburg20008consistency}. In such analyses, the optimal cost function is inherently tied to the chosen kernel and hence can be arbitrarily far from the ``ground truth.'' For instance, even an arbitrary clustering can be optimal for trivial kernels (constant or identity functions). 
Another approach to measure the performance of a clustering algorithm is by establishing recovery guarantees under distributional assumptions, sometimes known as \emph{planted models}. Distributional assumptions, or specifically (sub)-Gaussian mixture model assumption, is often considered in the theory of clustering. While learning a mixture of Gaussians has always been an important research problem, \citet{dasgupta1999learning}, for the time, presented a provable clustering algorithm to learn a mixture of \emph{high-dimensional} Gaussians. Theoretical research on learning high-dimensional Gaussians have ever since been highly significant, owing to the ubiquity of high-dimensional data in practice. Recent works in this direction provide \emph{phase transitions} for both clustering and parameter estimation of a mixture of high-dimensional Gaussians \parencite{banks2018information,ashtiani2018nearly}.

\citet{couillet2016kernel} initiated the theoretical study of kernel methods for high-dimensional Gaussian clustering, and in particular, presented the large sample behaviour of kernel spectral clustering in the regime where number of samples grow linearly with the data dimension. The statistical difficulty in this regime stems from the fact the Euclidean distance tends to be less informative in high dimensions and intra-cluster distances could be systematically larger than inter-cluster distances. \citet{yan2016robustness} generalised the problem setup to sub-Gaussian mixtures and derived sufficient conditions for achieving zero clustering error using convex relaxations of the kernel k-means objective. In both works, the analysis is restricted to computationally efficient clustering algorithms and the optimality of kernel methods, in terms of comparing necessary and sufficient conditions for clustering, is not addressed.

In this paper, we study the phase transition of the high-dimensional Gaussian clustering problem, similar to \citet{banks2018information}, where the number of samples is linear in the problem dimension. However, we focus on the case where one has access to only a kernel matrix. In other words, while the information-theoretic thresholds inherent to the Gaussian clustering problem, are expected to remain unchanged in the kernel setting with a non-trivial kernel, we prove that one can nearly achieve such thresholds using popular kernel methods. The \textbf{main contributions} in this paper are the following:
\\
\textbf{(1)} We identify the smallest separation between the means of latent clusters such that the clusters are statistically distinguishable under a kernel k-means objective in the sense of partial recovery, that is, error smaller than random guessing. Our result \textbf{matches the phase transition} for high-dimensional Gaussian clustering without kernels \parencite{banks2018information}. 
\\
\textbf{(2)} We analyse a common \textbf{semi-definite relaxation} of the kernel k-means objective and present sufficient conditions for partial recovery that \textbf{match, up to constant factors, the known spectral threshold} --- akin to the Kesten-Stigum threshold in the community detection under stochastic block model literature \parencite{baik2005phase,paul2007asymptotics}. 

Our main results obtained from the analysis of the two exponential-kernel clustering algorithms and the best known results for the same problem in a non-kernel setup are summarized in the table below. $k$ is the number of clusters, and $\alpha$ is the ratio of sample size to the data dimension which remains asymptotically finite in our setting.

The lower and the upper bounds are on the minimum separation of the clusters required to achieve partial recovery. The first column contains the bounds for the information-theoretic threshold. The second column contains the bounds corresponding to the computational class of poly-time algorithms. 
\begin{center}

{\renewcommand{\arraystretch}{1.7}
\begin{tabular}{  m{6em}  m{2.7cm} m{2.2cm}  } 
\toprule
& {\small\bf Information-theoretic limit} & {\small\bf Poly-time solvable} \\ 
\toprule
\small{Lower bounds} & $\sqrt{\frac{2(k-1) \log (k-1)}{\alpha}}$ & $ \qquad \frac{k-1}{\sqrt{\alpha}}$ \\ 
\hline
\small{Upper bounds (non-kernel)} & $2 \sqrt{\frac{k \log k}{\alpha}} + 2 \log k$ & $O(k-1 \vee \frac{k-1}{\sqrt{\alpha}})$ \\ 
\hline
{\color{blue}{\small{Upper bounds (kernel})}} & {\color{blue}$2 \sqrt{\frac{k \log k}{\alpha}} + 2 \log k$} & {\color{blue}$ \; \;  O(k \vee \frac{k}{\sqrt{\alpha}})$} \\ 
\bottomrule
\end{tabular}}
\end{center}

As noted in \citet{couillet2016kernel}, one requires a second-order analysis since first-order approximation of the kernel function does not suffice for the analysis in the high-dimensional setting. To this end, our proofs show that recent polynomial concentration inequalities \parencite{gotze2019concentration} can be useful for second-order analysis of kernel methods. 

\section{Background and Setting}
\textbf{Notation:} We denote denotes the set of natural numbers $\left \{1,2, \ldots ,k\right \}$ by $[k]$. For any matrix $A$, $\norm{A}_{F}$ refers to the Frobenius norm of the matrix. For any vector $x$, $\norm{x}$ refers to the Euclidean norm of the vector and $\mynorm{x}_1$ denotes the $l_1$ norm of the vector. $\mathbb{I}$ denotes the identity matrix. For any $A \in \mathbb{R}^{m \times m}$, $\norm{A}_{\infty \rightarrow 1}$ refers to the $\infty \rightarrow 1$ operator norm and defined as $ \sup \limits_{\substack{m}} (y^T A z)$. For any $n$ real numbers $\left \{ a_i \right \}_{i=1}^n$, $(a_1 \; \vee \; a_2 \; \ldots \vee \; a_n)$ refers to the maximum of the sequence: $\max \limits_i a_i$. For any random variable $x$, $\E x$ denotes the expectation of $x$.

\textbf{Setting:} Our setting is akin to the one used in \citet{banks2018information}, specifically due to the existence of a near-optimal phase transition for the information-theoretic threshold in the setting. We assume that the data is generated according to the following process. Let $k$ be the number of clusters. Then, $k$ points $\left \{ \mu'_1, \mu'_2, ..., \mu'_k \right \} \in \mathbb{R}^p$ are generated independently according to a normal distribution with mean $0$ and covariance $\frac{k}{k-1}\mathbb{I}_{p}$. The $k$ points are then centered by subtracting their sample mean from each entry and the resulting centered vectors are denoted by $\left \{ \mu_1, \mu_2, ..., \mu_k \right \}$. Let $m = \alpha p$ for some $\alpha > 0$ - a fixed parameter. Then for each $i \in [k]$, generate $\frac{m}{k}$ points from a normal distribution with mean $\sqrt{\frac{\rho}{p}}\mu_{i}$ and covariance matrix $\mathbb{I}$ for some fixed parameter $\rho > 0$. Observe that the parameter $\rho$ represents the separation between the clusters and can be treated as the parameter indicating the ``statistical ease'' with respect to the clustering problem or alternatively as the signal-to-noise ratio in this setting (we have an identity covariance matrix). We are interested in studying the large sample behaviour of clustering approaches in the high-dimensional setting: $m,p \rightarrow \infty$ and $\frac{m}{p} = O(1)$.

We denote the resulting set of $m$ points by $\left \{ x_1, x_2, ..., x_m\right \}$. Let $ \sigma :[m] \rightarrow [k]$ denote a balanced partition of $m$ points into $k$ clusters and let $\sigma_{*}$ denote the true partition: $\sigma_*(i) = s \textrm{ if } \mathbb{E} x_i = \sqrt{\frac{\rho}{p}} \mu_{s}$. Let $X^*$ denotes the ground truth clustering matrix defined as follows: 
\begin{equation*}
    X^*_{i,j}=
    \begin{cases}
      1 &\textrm{ if } \sigma_*(i) = \sigma_*(j)  \\
      0 &\textrm{  otherwise}.
    \end{cases}
\end{equation*}

For any arbitrary partition $\sigma$, define the $k \times k$ overlap matrix $\beta(\sigma,\sigma_{*})$, for each $s,t \in [k]$ as the fraction of all points assigned by $\sigma$ to the $s^{th}$ cluster \textbf{and} the fraction of all points assigned by $\sigma_{*}$ to the $t^{th}$ cluster,
\[ \beta(\sigma,\sigma_{*})_{s,t} = \frac{k \abs{\sigma^{-1}(s) \cap \sigma_*^{-1}(t) }}{m}.\]

Then $\norm{\beta(\sigma,\sigma_{*})}_{F}^2$ is a measure of similarity of the partition, $\sigma$ with the true partition, $\sigma_{*}$. Observe that if the partitions are completely uncorrelated, then $\beta$ is the constant matrix of $1/k$ and $\norm{\beta(\sigma,\sigma_{*})}_{F}^2 = 1$. If the partitions are identical up to permutations over the labels, then $\beta$ would be the permutation matrix and $\norm{\beta(\sigma,\sigma_{*})}_{F}^2 = k$.

Alternatively, for the sake of analytical tractability, we sometimes, use the quantity $err(\sigma,\sigma_{*})$ to denote the fraction of points misclassified by $\sigma$.
\[
err(\sigma,\sigma_{*}) = 1 - \frac{\max \limits_{\pi} \textrm{Trace}(\pi \beta(\sigma,\sigma_{*}))}{k}.
\]
where $\pi \beta$ refers to the matrix resulting from a permutation of $\beta$ over the cluster labels and the maximum is over all possible such permutations.

\textbf{Clustering with k-means:}
The clustering objective of the k-means procedure \parencite{pollard1981strong} is given as follows: 
\[\min \limits_{\sigma:[m] \rightarrow [k]} \sum \limits_{s = 1}^{k} \sum \limits_{i \in \sigma^{-1}(s)} \norm{x_i - \frac{k}{m}\sum \limits_{\sigma(j) = s} x_j}^2.\] 
This is equivalent to the following optimization problem: 
\[\max \limits_{\sigma:[m] \rightarrow [k]} \sum \limits_{s = 1}^{k} \sum \limits_{i,j \in \sigma^{-1}(s)} \inner{x_i}{x_j}.\] 
\textbf{Kernel k-means:}
For any partition $\sigma: [m] \rightarrow [k]$, Define, 
\[ \mathcal{F}(\sigma) = \sum \limits_{s = 1}^{k} \sum \limits_{\substack{i,j \in \sigma^{-1}(s)}} k(x_i,x_j). \]
Then, by the use of the kernel trick, we can formulate the kernel k-means clustering objective as follows: 
\begin{equation}
\label{prob:OptKmeans}
    \max \limits_{\sigma:[m] \rightarrow [k]} \mathcal{F}(\sigma)
\end{equation}
where $k:\mathbb{R}^p \times \mathbb{R}^p  \rightarrow \mathbb{R}$ is a kernel function. 
Minimizing this objective over all possible partitions is NP-hard \parencite{garey1982complexity,aloise2009np}. Several convex relaxations of the k-means procedure exist in literature. A well known Semi-definite program (SDP) relaxation of the kernel k-means \parencite{peng2007approximating} objective is given by:
\begin{align}
\max_{X} &\ \tr(KX) \label{eq:sdp1} \\
\text{s.t.,} &\  X\succeq 0, X\ge 0, \ X{\bf{1}}=\frac{m}{k} {\bf{1}},\ \text{diag}(X)={\bf{1}} \nonumber.
\end{align}
 where, $K$ refers to the kernel matrix for a given kernel function $k$: $K_{i,j} = k(x_i,x_j)$. This SDP can be solved in polynomial time. To obtain a partitioning $\hat{\sigma}$ of the data based on the optimal solution $\hat{X}$ of the SDP, a 7-approximate k-medians's procedure \parencite{charikar2002constant} is applied on the rows of the matrix $\hat{X}$ in a similar fashion as \citet{fei2018exponential}. We denote the partition inferred by this procedure as $\hat{\sigma}$. The details of the k-median procedure can be found in \citet[Algorithm 1]{fei2018exponential}. 
 
 \textbf{Choice of kernel function:} In our analysis, for simplicity, we considered the following dot product kernel: $k(x,y) = f(\inner{x}{y}) = \exp \left (\frac{\inner{x}{y}}{p} \right )$. However, the analysis can be extended to a larger class of dot-product kernels. 

\section{Our Results}

We denote the upper bound on the information-theoretic threshold derived from the analysis of the maximum likelihood estimator, in the non-kernel setting, as $\rho^{upper}_{linear\,NP}$ and the best known lower bound as $\rho^{lower}_{NP}$. Similarly, we denote the upper bound from the analysis of the NP-hard kernel k-means procedure as $\rho^{upper}_{kernel\,NP}$. 

The central question we address in this section is the following: Does any exponential-kernel clustering procedure asymptotically achieve information-theoretic optimality for high-dimensional clustering? 
\[ \rho^{upper}_{kernel\,NP} \stackrel{?}{=} \rho^{lower}_{NP}.\]
The maximum likelihood estimator in the non-kernel setting is already known to achieve near optimality in an information-theoretic sense. Therefore, our principal objective can be rephrased, in essence, as: Is the exponential kernel more(or less) informative than the linear kernel?
\[\rho^{upper}_{kernel\,NP} \stackrel{?}{\leq} \rho^{upper}_{linear\,NP}.\]

 As noted earlier, optimizing the kernel k-means objective is NP-hard. Therefore, for practical significance, it is also interesting to understand the information-theoretic optimality of kernels via kernelized, computationally efficient clustering algorithms. To this end, we analyze the kernel SDP given in (\ref{eq:sdp1}). It has been observed in several clustering and community detection problems that the parameter space of the signal to noise ratio $\rho$ where polynomial time algorithms are known to succeed is, typically, strictly above the information-theoretic threshold. To evaluate if there is any information loss, due to the use of kernels in polynomial time clustering algorithms, we compare the signal-to-noise ratio above which the kernel SDP can provably recover the true clustering, $\rho^{upper}_{kernel\,P}$, with the known spectral threshold, $\rho^{lower}_{P}$ below which no known poly-time algorithm is known to succeed. We also compare $\rho^{upper}_{kernel\,P}$
 to the upper bound($\rho^{upper}_{linear\,P}$) derived from the analysis of a similar semidefinite relaxation of the linear k-means algorithm. 
\[ \rho^{upper}_{kernel\,P} \stackrel{?}{=} \rho^{lower}_{P} \qquad ; \rho^{upper}_{kernel\,P} \stackrel{?}{=} \rho^{upper}_{linear\,P}.\]
We pictorially demonstrate all our results in Figure \ref{fig:stat-comp}. 
\setlength{\belowcaptionskip}{-12pt}
\begin{figure}
    \centering
    \includegraphics[width=0.49\textwidth]{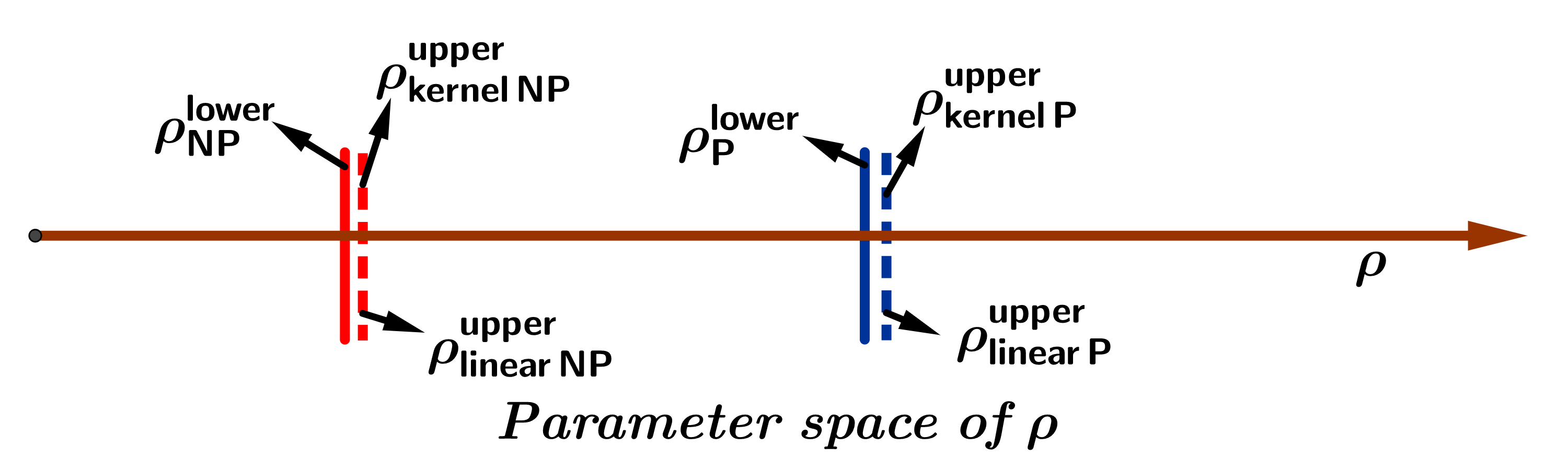}
    \caption{Our upper bounds are near optimal and exactly match those of the non-kernel setting.}
    \label{fig:stat-comp}
\end{figure}
\subsection{Optimality of kernel k-means}

The following lower and upper bounds, $\rho^{lower}_{NP}$ and $\rho^{upper}_{linear\,NP}$ respectively on the information-theoretic threshold appeared in \citet{banks2018information}.
\begin{align}
		\rho^{upper}_{linear\,NP} &= 2\sqrt{ \frac{k \log k}{\alpha} } + 2\log k. 
	\end{align}
	\begin{align}
	\rho^{lower}_{NP} &= \begin{cases}
		          \sqrt{1/\alpha} & k=2 \\
		          \sqrt{\frac{2(k-1)\log(k-1)}{\alpha}} & k\geq 3 \,.
		          \end{cases}
\end{align}

We analyze the performance of the kernel k-means clustering algorithm and give the following upper bounds on the information-theoretic threshold:
\begin{theorem}[\textbf{Optimality of kernel k-means}]
\label{thm:npHard}
Let,
\begin{align}
		\rho^{upper}_{kernel\,NP} &= 2\sqrt{ \frac{k \log k}{\alpha} } + 2\log k. 
	\end{align}
	Then if $\rho > \rho^{upper}_{kernel\,NP}$, then it is asymptotically possible to recover the true partition.
\end{theorem}
Our results show that there is no loss of information incurred due to the use of the exponential kernel function in high-dimensional Gaussian clustering. This also matches the known information-theoretic lower bounds up to a factor of $\sqrt{2}$ when the number of clusters $k$ is large \parencite{banks2018information}.

\textbf{Overview of the analysis:} On a high level, the main line of argumentation of the proof is similar to the one in \citet{banks2018information}. However, note that their analysis only holds for the linear k-means algorithm and extending the analysis to a second order expansion of the exponential-kernel k-means objective is considerably more complex and requires a different set of mathematical tools and techniques (see Section \ref{sec:proofnpHard}). 

We consider the distribution of the objective of kernel k-means $\mathcal{F}(\sigma)$ as a function of the partition $\sigma$. 
We show that above the aforementioned threshold $\rho^{upper}_{kernel\,NP}$, with high probability, the distribution of $\mathcal{F}(\sigma_{*})$ is disjoint with and higher than that of the distribution of $\max \limits_{\substack{\sigma: \norm{\beta(\sigma,\sigma_{*})}^2_F  \\ \le 1+(k-1) \epsilon}} \mathcal{F}(\sigma)$, where $\epsilon >0$ is an arbitrarily small constant. 

Let $\tilde{\sigma}$ denote the optimal solution to (\ref{prob:OptKmeans}). Since, by definition, $\mathcal{F}(\tilde{\sigma}) \geq \mathcal{F}(\sigma_{*})$, it follows that the support of the distribution of $\mathcal{F}(\tilde{\sigma})$ is disjoint with and higher than that of the distribution of $\max \limits_{\substack{\sigma: \norm{\beta(\sigma,\sigma_{*})}^2_F  \\ \le 1+(k-1) \epsilon}} \mathcal{F}(\sigma)$.
\subsection{Optimality of kernel SDP}
 The following phase transition for spectral methods can be inferred from \citet{paul2007asymptotics,baik2005phase} and appeared in \citet{banks2018information}.
 \begin{align*}
   \rho^{lower}_{P} = \frac{k-1}{\sqrt{\alpha}}.
 \end{align*}
 We give the following upper bound on the threshold below which no known computationally efficient polynomial clustering approaches are guaranteed to succeed. 
 \begin{theorem}[\textbf{Optimality of kernel SDP}]
 \label{thm:conv}
 Let, 
 \begin{align}
     \rho^{upper}_{kernel\,P} =  c k \left ( 1 \vee \frac{1}{\sqrt{\alpha}} \right ).
 \end{align}
 for some fixed constant $c >0$. Then, if $\rho > \rho^{upper}_{kernel\,P}$ kernel SDP can asymptotically recover the true partition.
 \end{theorem}

 Our results match the spectral threshold up to constant factors of approximation for large $k$. Our result also matches with the known upper bound ($\rho^{upper}_{linear\,P}$) for partial recovery via the linear k-means clustering procedure \parencite{giraud2018partial} up to a factor of $\frac{(k-1)}{k}$. In agreement with the established conjecture, it is also evident from our results that the threshold at which a computationally efficient kernel clustering procedure can be guaranteed to succeed is strictly above the information-theoretic threshold. They differ by an order of $\sqrt{\frac{k}{\log k}}$.

\textbf{Overview of the analysis:} 
Denote $\kappa = e^{\tau} e^{\frac{C_0 \log p}{\sqrt{p}}}$ We define the matrix $\Ktilde$ - which depends on the population parameters of the data distribution - as follows: 
\[ \Ktilde(i,j) = f(0) + \begin{cases} \frac{f^{'}(0) \rho \inner{\mu_{i}}{\mu_{j}}}{p^2} + \frac{ \kappa \rho^2 \inner{\mu_i}{\mu_j}^2}{p^4} + \frac{\kappa}{p} & \textrm{if } i \neq j \\
\frac{f^{'}(0)(p^2 + \rho \norm{\mu_{i}}^2)}{p^2} + \frac{ \kappa (p^2 + \rho \norm{\mu_{i}}^2)^2}{p^4} + \frac{\kappa}{p}  &\textrm{otherwise}.
\end{cases}\]

We show that the kernel matrix $K$ concentrates around $\Ktilde$ in the $\infty \rightarrow 1$ operator norm.

Let $\hat{X}$ denote the optimal solution to (\ref{eq:sdp1}). Then, using Grothendieck's inequality \parencite{grothendieck1956resume}, we derive an upper bound on $\mynorm{\hat{X} - X^*}_{1}$ in terms of $\mynorm{K - \Ktilde}_{\infty \rightarrow 1}$. Since $\hat{X}$ is not a partition matrix, we need a procedure that can infer a partition from $\hat{X}$. 
We use the 7-approximate k-median's procedure \parencite{fei2018exponential} on the rows of $\hat{X}$ to infer a partition $\hat{\sigma}$. Then \citet{fei2018exponential} showed that the fraction of mis-classified vertices by the partition $\hat{\sigma}$ denoted by $err(\sigma_{*}, \hat{\sigma})$ can be upper bounded by a constant factor of $\frac{\norm{\hat{X} - X^*}_{1}}{\norm{X^*}_1}$.

Thereby, we show that for $\rho > \rho^{upper}_{kernel\,P}$, the fraction of misclassified points $err(\hat{\sigma},\sigma_*) < (1 - 1/k)$, which is the condition required for partial recovery.
 
SDP's such as the one defined in (\ref{eq:sdp1}) have been analyzed using the Grothendieck's inequality approach in community detection literature for stochastic block models \parencite{guedon2016community}. However, the main technical challenges of our analysis lie in the choice of appropriate $\Ktilde$ and showing that the matrix $K$ concentrates around $\Ktilde$ in the $\infty \rightarrow 1$ operator norm.
Establishing the concentration results for $\mynorm{K - \Ktilde}_{\infty \rightarrow 1}$ is considerably harder compared to the analysis of similar quantities based on the adjacency matrix of a network generated from a stochastic block model. Unlike in the case of adjacency matrices, the entries of the kernel matrix encode dependencies between the data points and hence most classical concentration tools from random matrix theory fall short in the analysis of kernel matrices. Also the RKHS corresponding to the exponential class of kernel functions is infinite dimensional and hence concentration inequalities that depend on the dimension of the feature space are also not applicable for analyzing functions of kernel entries. 

To this end, we demonstrate that the polynomial concentration inequalities for exponentially decaying random variables in \citet{gotze2019concentration} can be used to analyze an entry-wise second order approximation of the kernel matrix. 
We make some further remarks about our proof, and possibilities for improving the result.

 \textbf{Remark 1:} Finer upper bounds on $\inner{K - \Ktilde}{X^* - \hat{X}}$ can be obtained by applying an analysis similar to \citet{fei2018exponential} to obtain better error rates. However, our bounds on $\rho$, essentially remain the same - which is the main emphasis of this paper. 

\textbf{Remark 2:} The choice of $\Ktilde$ can further be refined in the second order terms without changing the results of our analysis. 

\textbf{Remark 3:} One could, alternatively, infer a partition by applying the k-means procedure on the rows of the eigenvectors of $\hat{X}$. Using Davis-Khan's theorem \parencite{yu2014useful}, one may similarly upper bound the fraction of misclassified nodes by a constant factor of $\frac{\mynorm{\hat{X} - X^*}_{1}}{\mynorm{X^*}_1}$. This approach gives a slightly worse approximation constant.

\section{Proofs}

\subsection{Proof of Theorem \ref{thm:npHard}}
\label{sec:proofnpHard}

\textbf{Overview of the technical steps:}
Let $\epsilon >0$ be an arbitrarily small constant. The two main ingredients required to establish conditions of recovery are as follows:
\begin{center}
   \begin{itemize}
    \item Upper tail bounds for $\max \limits_{\substack{\sigma: \norm{\beta(\sigma,\sigma_{*})}^2_F  \\ \le 1+(k-1) \epsilon}} \mathcal{F}(\sigma)$.
      \item  Lower tail estimates for the distribution of $\mathcal{F}(\sigma_{*})$.
\end{itemize} 
\end{center}

To obtain these bounds, for any fixed $\sigma$, we first apply the Taylor's theorem - with mean value form of the reminder - to obtain a 2nd order polynomial approximation of each of the kernel entry and obtain a tight lower bound $\mathcal{F}_l(\sigma)$ and an upper bound $\mathcal{F}_{u}(\sigma)$ on $\mathcal{F}(\sigma)$.

For any fixed $\sigma$ such that $\norm{\beta(\sigma,\sigma_{*})}^2_F  \le 1+(k-1) \epsilon$ we compute $\mathcal{F}_{u}(\sigma)$ which is a 4th order polynomial of normally distributed random variables and carefully upper bound all the terms of this polynomial using various known concentration results in literature. By an union bound over all such partitions, we obtain upper tail bounds for $\max \limits_{\substack{\sigma: \norm{\beta(\sigma,\sigma_{*})}^2_F \le 1+(k-1) \epsilon}} \mathcal{F}_{u}(\sigma)$. Similarly, we compute $\mathcal{F}_{l}(\sigma_{*})$ which is a 4th order polynomial of normally distributed random variables and obtain lower bounds for all the involved terms.

Therefore, we obtain:
\begin{align*}
   \mathcal{F}(\sigma_{*}) &\geq \mathcal{F}_{l}(\sigma_{*}) > \omega_{l} \quad \textrm{ and }
   \max \limits_{\substack{\sigma: \norm{\beta(\sigma,\sigma_{*})}^2_F  \\ \le 1+(k-1) \epsilon}} \mathcal{F}(\sigma) \leq \max \limits_{\substack{\sigma: \norm{\beta(\sigma,\sigma_{*})}^2_F  \\ \le 1+(k-1) \epsilon}} \mathcal{F}_{u}(\sigma) \leq  \omega_{u}.
\end{align*}

By comparing $\omega_{u}$ and $\omega_{l}$, we obtain the conditions on $\rho$ under which $\omega_{l} \geq \omega_{u}$.

\textbf{Notation:}
We use the following notation for some recurring terms for improved readability.

For any $i,j \in [m]$, set $\tau = \frac{\mathbb{E}\norm{x_i}^2}{p} = 1 + O(1/p)$. For any $\sigma$, we use the following notation:
\begin{align*}
    Q_{1 \sigma} = \frac{k}{m} \slim_{s \in [k]} \slim_{i,j \in \sigma^{-1}(s)} \frac{\inner{x_i}{x_j}}{p} \; \; ; &Q_{2 \sigma} = \frac{k}{m} \slim_{s \in [k]} \slim_{i,j \in \sigma^{-1}(s)} \frac{\inner{x_i}{x_j}^2}{p^2} \\
    \qquad \qquad  Q_{3} \; \; = \frac{k (f'(\tau) - 1)}{m} \slim_{i \in [m]} \big(\frac{\mynorm{x_i}^2}{p} - \tau \big ) \; \; ; &Q_{4} = \frac{k (e^{\tau} - 1)}{2 m} \slim_{i \in [m]} \big (\frac{\mynorm{x_i}^2}{p} - \tau \big)^2; \; \; Q_{5} = \slim_{i \in [m]} \frac{k \tau \mynorm{x_i}^2}{p}; 
\end{align*}
and set,
  $ \gamma_{\max} = \exp \left ( \frac{C \log p}{\sqrt{p}} \right ); \gamma_{\min} = \exp \left ( - \frac{C \log p}{\sqrt{p}} \right )$
for some constant $C > 0$.

All the lemmas we state below hold with high probability ($1 - \Omega(\frac{1}{p})$) and the proofs of all the lemmas are provided in the supplementary.

\textbf{Outline of the proof:}

Recall that for any partition $\sigma:[m] \rightarrow [k]$, $\mathcal{F}(\sigma) = \frac{k}{m} \sum \limits_{s \in [k]} \sum \limits_{i,j \in \sigma^{-1}(s)} k(x_i,x_j)$.
\begin{lemma}[\textbf{Upper and lower bounds for inner products}]
\label{lemma:MaxDpMaxNorm}
\begin{align*}
     \max \limits_{i, j} \frac{\abs{\inner{x_i}{x_j}}}{p} =
    \tau \mathbf{1}_{i=j} +  O\left(\frac{\log p}{\sqrt{p}}\right)  \; \; ; 
    \min \limits_{i, j} \frac{\inner{x_i}{x_j}}{p} = \tau \mathbf{1}_{i=j}  + \Omega \left(- \frac{\log p}{\sqrt{p}}\right) .
\end{align*}
\end{lemma}

By a second order Taylor expansion of each $k(x_i,x_j)$ where $i \neq j$ around $0$ and expanding each $k(x_i,x_i)$ around $\tau $, 

and using Lemma \ref{lemma:MaxDpMaxNorm}, for any $\sigma$, we can write:
\begin{equation*}
 \mathcal{F}(\sigma) \leq \mathcal{F}_{u}(\sigma) = Q_{1 \sigma} + \gamMax Q_{2 \sigma} +  Q_{3} + \gamma_{\max}(Q_{4} -  Q_{5}) +  k f(\tau) + (m-k)f(0) + \frac{k \gamMax \tau^2}{2} .
\end{equation*}
and 
\begin{equation*}
 \mathcal{F}(\sigma) \geq \mathcal{F}_{l}(\sigma) = Q_{1 \sigma} + \gamMin Q_{2 \sigma} +  Q_{3} + \gamMin(Q_{4} -  Q_{5}) +  k f(\tau) + (m-k)f(0) + \frac{k \gamMin \tau^2}{2} .
\end{equation*}
\textbf{Upper bounds for $\max_{\substack{\sigma: \norm{\beta(\sigma,\sigma_{*})}^2_F  \\ \le 1+(k-1) \epsilon}} \mathcal{F}(\sigma)$:}
We derive upper bounds for all the terms that constitute $\mathcal{F}_{u}(\sigma)$, which simultaneously hold for all $\sigma: \norm{\beta(\sigma,\sigma_{*})}^2_F  \le 1+(k-1) \epsilon$.

\begin{lemma}[\textbf{Upper bounds - $Q_{1 \sigma}$, $Q_{5}$}]
\label{lemma:UBQ1Q5}
\begin{equation*}
	\max_{\substack{\sigma: \norm{\beta(\sigma,\sigma_{*})}^2_F  \\ \le 1+(k-1) \epsilon}} Q_{1\sigma} \leq k +  \alpha\rho\epsilon +  2(1+\epsilon) \alpha \log k + 2 \sqrt{ (1+\epsilon) \left(k + 2 \alpha\rho \epsilon \right)\alpha \log k }
	 + O(\sqrt{\log p/p} ).
\end{equation*}
\begin{equation*}
    \max_{\substack{\sigma: \norm{\beta(\sigma,\sigma_{*})}^2_F  \\ \le 1+(k-1) \epsilon}}  \gamma_{\max}Q_{5}  \leq -\frac{k\gamma_{\max}\tau}{mp} \left ( mp + p \alpha \rho - 2 \sqrt{\left ( mp + 2p\alpha\rho \right ) \log p} \right ).
\end{equation*}
\end{lemma}
\begin{proof}[Proof (sketch)]
The terms $Q_{1 \sigma}$ and $Q_{5} $ can be expressed as sums of independent non-central chi-squared random variables and applying the known upper tail bounds for such sums, followed by a union bound over all $\sigma: \norm{\beta(\sigma,\sigma_{*})}^2_F  \le 1+(k-1) \epsilon$, we have the results from Lemma \ref{lemma:UBQ1Q5}. 
\end{proof}

\begin{lemma}[\textbf{Upper bound for $Q_{2 \sigma}$}]
\label{lemma:UBQ2}
\begin{equation*}
    \max_{\substack{\sigma: \norm{\beta(\sigma,\sigma_{*})}^2_F  \\ \le 1+(k-1) \epsilon}} \gamma_{\max}  Q_{2\sigma} \leq \gamma_{\max} \left (1 + \frac{1}{k} + O(\frac{1}{p}) \right ) + C_{2} \gamma_{\max} O \left ( \sqrt{\frac{\alpha}{p}} \vee  \alpha \sqrt{\frac{\alpha}{p}} \vee \sqrt{\frac{1}{\alpha p}} \right ).
\end{equation*}
for some constant $C_{2} > 0$.
\end{lemma}
\begin{proof}[Proof (sketch)]
Controlling the typical behavior of the term $Q_{2\sigma}$ is the most demanding part of the proof. We use the concentration results established for polynomials of sub-Gaussian random variables (see the supplementary for a definition) in \citet{gotze2019concentration} to establish the result in Lemma \ref{lemma:UBQ2}.
\end{proof}
\begin{lemma}[\textbf{Upper bound for $Q_{4}$}]
\label{lemma:UpperBoundQ4sigma}
\begin{equation*}
     \max_{\substack{\sigma: \norm{\beta(\sigma,\sigma_{*})}^2_F  \\ \le 1+(k-1) \epsilon}} Q_{4} \leq C_{4} k \gamMax (e^{\tau}-1) (\log p)^2/2p.
\end{equation*}
for some constant $C_{4} > 0$.
\end{lemma}
\begin{proof}[Proof (sketch)]
The term $Q_{4}$ is small relative to the other terms and hence a crude upper bound based on the inequality: For any two vectors $a,b$, $\sum \limits_{i = 1}^{n} a_{i} \cdot b_{i} \leq \sup \limits_{i \in [n]} \abs{b_{i}} \cdot \sum \limits_{i = 1}^{n} \abs{a_{i}}$, followed by an application of Lemma \ref{lemma:MaxDpMaxNorm} suffices to establish the behavior of this term. 
\end{proof}

\textbf{Lower bounds for $\mathcal{F}(\sigma_{*})$.}
Similarly, we derive lower bounds for all terms that arise in $\mathcal{F}_{l}(\sigma_{*})$.

\begin{lemma}[\textbf{Lower bound for $Q_{1 \sigma_{*}}$ and $Q_{5}$}]
\label{lemma:lowerBoundsQ1Q5}
\begin{align*}
   &Q_{1 \sigma_{*}} >  k + \alpha \rho - O(\sqrt{\log p/p}). \\
    - Q_{5} >  -\frac{k\gamma_{\min}\tau}{mp} &\left ( mp + p \alpha \rho +2\log p \right )  -\frac{k\gamma_{\min}\tau}{mp} \left ( 2 \sqrt{\left ( mp + 2p\alpha\rho \right )\log p} \right ).
\end{align*}
\end{lemma}
\begin{proof}[Proof (sketch)]
From upper tail estimates for sums of non-central chi-squared random variables, we establish the result of Lemma \ref{lemma:lowerBoundsQ1Q5}.
\end{proof}

\begin{lemma}[\textbf{Lower bound for $Q_{2 \sigma_{*}}$}]
\label{lemma:LBQ2sigma}
\begin{equation*}
      \gamma_{\min} Q_{2 \sigma_{*}} >  \gamma_{\min} \left (1 + \frac{1}{k} + O(\frac{1}{p}) \right ) - C_{2} \gamma_{\min} \left ( \sqrt{\frac{\log p}{p^2}} \vee  \alpha \sqrt{\frac{\log p}{p^2}} \right ).
\end{equation*}
\end{lemma}
\begin{proof}[Proof (sketch)]
$Q_{2 \sigma_{*}}$, as discussed earlier, is a $4^{th}$ order polynomial of sub-Gaussian random variables (see the supplementary for a definition). Therefore from lower tail estimates for polynomials of sub-Gaussian random variables in \citet{gotze2019concentration}, we establish the result in Lemma \ref{lemma:LBQ2sigma}.
\end{proof}

Since $Q_{4}$ is a smaller term, the following lower bound suffices to control its behavior.
\begin{equation}
\label{eqn:LBQ4}
   Q_{4} > 0.
\end{equation}
 Observe that $\gamMax - \gamMin \xrightarrow{O(1/\sqrt{p})} 0$ with increasing $p$. From Lemmas \ref{lemma:UBQ1Q5} to \ref{lemma:LBQ2sigma} and Equation \ref{eqn:LBQ4}, we obtain that for $\rho > 2\sqrt{k \log k / \alpha} + 2 \log k$, for large enough $p$, with high probability, $  \max \limits_{\sigma} \mathcal{F}(\sigma) \geq \mathcal{F}(\sigma_{*}) \geq \max_{\substack{\sigma: \norm{\beta(\sigma,\sigma_{*})}^2_F \le 1+(k-1) \epsilon}} \mathcal{F}(\sigma)$.

\subsection{Proof of Theorem \ref{thm:conv}}

\textbf{Overview of the technical steps:}
\begin{itemize}
    \item We define a matrix $\Ktilde$ which relies on the model parameters of the data distribution. 
    
    \item We upper bound the $l_1$ norm of the difference between the ground truth clustering matrix and the optimal solution of the SDP $\mynorm{\hat{X} - X^*}_1$ by a constant factor of the inner product between $K - \Ktilde$ and $ X^* - \hat{X}$ : $\inner{K - \Ktilde}{X^* - \hat{X}}$.
    
    \item We use the Grothendieck's inequality to upper bound $\inner{K - \Ktilde}{X^* - \hat{X}}$ by a constant factor of $\mynorm{K - \Ktilde}_{\infty \rightarrow 1}$.
    
    \item We establish the upper tail estimates of the deviation of the kernel matrix $K$ from $\Ktilde$ in the $\infty \rightarrow 1$ norm.
    
    \item Thereby, we have an upper bound on $\mynorm{\hat{X} - X^*}_1$ which translates to an upper bound on $err(\hat{\sigma},\sigma_*)$. By setting $err(\hat{\sigma},\sigma_*) < 1-1/k$, we derive the desired conditions on $\rho$.
\end{itemize}

\textbf{Notation:} Denote $\kappa = e^{\tau} e^{\frac{C_0 \log p}{\sqrt{p}}}$. For ease of notation, we define the $m \times m$ matrices $R^{(1)}$ and $R^{(2)}$ as follows:
\[
R^{(1)}_{i,j} =  \begin{cases}
\frac{\inner{x_i}{x_j}}{p} - \frac{\rho \inner{\mu_i}{\mu_j}}{p^2} & \textrm{if } i \neq j, \\
\frac{\norm{x_i}^2}{p} - \frac{p^2 + \rho \norm{\mu_i}^2}{p^2} & \textrm{otherwise.}
\end{cases}
\]
\[
R^{(2)}_{i,j} =  \begin{cases}
\frac{\inner{x_i}{x_j}^2}{p^2} - \frac{\rho^2 \inner{\mu_i}{\mu_j}^2}{p^4} - \frac{1}{p} & \textrm{if } i \neq j, \\
\frac{\norm{x_i}^4}{p^2} - \frac{(p^2 + \rho \norm{\mu_i}^2)^2}{p^4} - \frac{1}{p} & \textrm{otherwise.}
\end{cases}
\]
All the lemmas hold with high probability: with probability $1 - \Omega(1/p)$ and the proofs of the lemmas are provided in the supplementary.

\textbf{Outline of the proof:}
We begin by establishing the following upper bound on $\mynorm{X^* - \hat{X}}_1$.
\begin{lemma}[\textbf{Upper bound on $\mynorm{X^* - \hat{X}}_1$}]
\label{lemma:boundl1Norm}
\begin{align*}
     \mynorm{X^* - \hat{X}}_1 \leq \frac{2 \inner{\Ktilde}{X^* - \hat{X}}}{\frac{\rho}{p} \left ( \frac{k}{k-1} + O(\sqrt{\log p / p}) + \kappa \rho O(\frac{1}{p}) \right )}.
\end{align*}
\end{lemma}
Observe that, by definition, $\inner{K}{\hat{X}} \geq \inner{K}{X^{*}} \implies \inner{K}{\hat{X} - X^*} \geq 0$, therefore,

\begin{equation*}
    \inner{K}{X^* - \hat{X}} \leq \inner{K-\Ktilde}{X^* - \hat{X}} \leq 2 \sup \limits_{\substack{X \succeq 0 \\ diag(X) \leq 1}} \abs{\inner{K-\Ktilde}{X}}.
\end{equation*}
Using Grothendieck's inequality \citep{grothendieck1956resume}, we arrive at the following (see the appendix for a statement of Grothendieck's inequality):

\begin{equation*}
    2 \sup \limits_{\substack{X \succeq 0 \\ diag(X) \leq 1}} \abs{\inner{K-\Ktilde}{X}} \leq K_{G} \mynorm{K - \Ktilde}_{\infty \rightarrow 1}. 
\end{equation*}
where $K_{G} \approx 1.783$ is the Grothendieck's constant.
For any pair of fixed vectors $z,y \in \left \{ \pm 1 \right \}^m$, by a 2nd order Taylor's expansion of each $K_{i,j}$ around $0$, and applying the result from Lemma \ref{lemma:MaxDpMaxNorm}, we can see that:
\begin{equation}
      y^T(K - \Ktilde)z \leq y^T(R^{(1)} + \kappa R^{(2)})z.
\end{equation}
\begin{lemma}[\textbf{Upper bounds for $R^{(1)}$}]
\label{lemma:UBConvFO}
\begin{equation*}
   \sup \limits_{z,y \in \left \{ \pm \right \}^n} y^T R^{(1)}z \leq  C_{1} \alpha \left (\sqrt{mp}  \vee  m \right ),
\end{equation*}
for some constant $C_1 > 0$.
\end{lemma}
\begin{proof}[Proof (sketch):]
Linear combinations of entries of the matrix $R^{(1)}$ can be re-written as sums of independent sub-exponential random variables (see the supplementary for a definition). By an application of Bernstein's inequality for each fixed $\left \{ z,y \in \pm 1 \right \}^m$, followed by an union bound over all possible $z,y$ we establish the result.
\end{proof}
\begin{lemma}[\textbf{Upper bounds for $R^{(2)}$}]
\label{lemma:UBConvSO}
\begin{equation*}
    \sup \limits_{\left \{z,y \in \pm 1 \right \}^m} \kappa 
    y^T R^{(2)} z
    \leq \frac{ C'_2 \kappa}{p^2}(\rho m(m-1) + m + (mp \sqrt{m} \; \vee \; m^2 \sqrt{m} \; \vee \; p^2 \sqrt{m})).
\end{equation*}
for some constant $C'_{2} > 0$.
\end{lemma}
\begin{proof}
In order to bound the linear combinations of entries of $R^{(2)}$, for each fixed $\left \{ z,y \in \pm 1 \right \}^m$ we apply the concentration results for polynomials of independent sub-Gaussian random variables. \parencite{gotze2019concentration}. In order to bound the maximum of the second order terms over all $\left \{ z,y \in \pm 1 \right \}^m$, we use the union bound.
\end{proof}
From Lemmas \ref{lemma:boundl1Norm}, \ref{lemma:UBConvFO} and \ref{lemma:UBConvSO}, we have that:
\begin{equation*}
    \mynorm{\hat{X} - X^*}_{1} \leq \frac{2 C^{'}}{\phi p} \left ( (m^2 / \sqrt{\alpha} \; \vee m^2) + \frac{\kappa}{p} (\rho m^2 + O(m) ) \right ) +  \frac{2 C^{'} \kappa}{\phi p^2} \left (mp \sqrt{m} \; \vee \; m^2 \sqrt{m} \; \vee \; p^2 \sqrt{m}\right ),
\end{equation*}
where $\phi = \frac{\rho}{p} \left ( \frac{k}{k-1} + O(\sqrt{\log p / p}) + \kappa \rho O(\frac{1}{p}) \right )$, for some constant, $C^{'} > 0 $.

Let $\hat{\sigma}$ be the partition generated by applying the $\eta$-approximate k-median's procedure on $\hat{X}$.
\begin{proposition}[\textbf{Fraction of misclassified nodes} \parencite{fei2018exponential}]
\label{prop:UBerr}
 The fraction of mis-classified points corresponding to the partition $\hat{\sigma}$:
\[ err(\hat{\sigma},\sigma_{*}) \leq 2(1 + 2 \eta) \frac{\mynorm{\hat{X} - X^*}_1}{\mynorm{X^*}_1}.\]
\end{proposition}
Observe that $\mynorm{X^*}_{1} = \frac{m^2}{k}$. Applying the result from Proposition \ref{prop:UBerr}, we have that for large enough $p$, if $\rho \gtrsim k (\frac{1}{\sqrt{\alpha}} \; \vee \; 1)$, $err(\hat{\sigma}, \sigma_{*}) < 1 - \frac{1}{k}$.
 
\section{Discussion}

In this paper, we study the large sample behaviour of the kernel k-means algorithm for high-dimensional clustering. The principal focus lies in investigating the information-theoretic optimality of the kernel k-means procedure. Recent works have demonstrated that the linear k-means algorithm is near optimal in this sense. Therefore another aspect of our work resides in understanding the informativeness of specific kernels for high-dimensional clustering in relation to the linear kernel. A thorough understanding of these aspects is fundamental to the use of kernels in any unsupervised high-dimensional learning problem.

We also study the large sample behaviour of a popular semi-definite relaxation of the kernel k-means objective. We emphasize on optimality and informativeness of kernels in computationally efficient algorithms for high-dimensional clustering. A widely believed conjecture in clustering literature, with support from well founded theoretical evidence is that computationally efficient algorithms are sub-optimal in an information-theoretic sense. Therefore, in this paper, we consider the SDP to be information-theoretically optimal if its optimal in the class of computationally efficient algorithms. The best known result for this class arises from the well known spectral threshold in this setting. 

We show that both the algorithms are near optimal in their computational class and as a consequence also demonstrate that their is no loss of information incurred by the use of the exponential kernel over the linear kernel. By virtue of our proofs, we also demonstrate that the recent polynomial concentration inequalities for random variables with exponentially decaying tails can aid in the analysis of higher order kernel approximations.

Furthermore, this line of analysis can be extended to other empirically popular kernels. In particular, since the squared distance is known to be less informative in high dimensions, it would be interesting to investigate the informativeness of the popular Gaussian kernel which relies on the square distances.

\subsubsection*{Acknowledgements}

\printbibliography

@inproceedings{ashtiani2018nearly,
  title={Nearly tight sample complexity bounds for learning mixtures of Gaussians via sample compression schemes},
  author={Ashtiani, Hassan and Ben-David, Shai and Harvey, Nicholas and Liaw, Christopher and Mehrabian, Abbas and Plan, Yaniv},
  booktitle={Advances in Neural Information Processing Systems},
  pages={3412--3421},
  year={2018}
}

@article{banks2018information,
  title={Information-theoretic bounds and phase transitions in clustering, sparse PCA, and submatrix localization},
  author={Banks, Jess and Moore, Cristopher and Vershynin, Roman and Verzelen, Nicolas and Xu, Jiaming},
  journal={IEEE Transactions on Information Theory},
  volume={64},
  number={7},
  pages={4872--4894},
  year={2018},
  publisher={IEEE}
}

@inproceedings{dasgupta1999learning,
  title={Learning mixtures of Gaussians},
  author={Dasgupta, Sanjoy},
  booktitle={40th Annual Symposium on Foundations of Computer Science},
  pages={634--644},
  year={1999},
  organization={IEEE}
}

@article{wang2019scalable,
  title={Scalable Kernel K-Means Clustering with Nystr{\"o}m Approximation: Relative-Error Bounds},
  author={Wang, Shusen and Gittens, Alex and Mahoney, Michael},
  journal={Journal of Machine Learning Research},
  volume={20},
  pages={1--49},
  year={2019}
}

@article{vonluxburg20008consistency,
  title={Consistency of spectral clustering},
  author={von Luxburg, Ulrike and Belkin, Mikhail and Bousquet, Olivier},
  journal={The Annals of Statistics},
  pages={555--586},
  year={2008},
  publisher={JSTOR}
}

@article{couillet2016kernel,
  title={Kernel spectral clustering of large dimensional data},
  author={Couillet, Romain and Benaych-Georges, Florent},
  journal={Electronic Journal of Statistics},
  volume={10},
  number={1},
  pages={1393--1454},
  year={2016},
  publisher={The Institute of Mathematical Statistics and the Bernoulli Society}
}

@article{mai2018random,
  title={A random matrix analysis and improvement of semi-supervised learning for large dimensional data},
  author={Mai, Xiaoyi and Couillet, Romain},
  journal={The Journal of Machine Learning Research},
  volume={19},
  number={1},
  pages={3074--3100},
  year={2018},
  publisher={JMLR. org}
}

@article{mendelson2010regularization,
  title={Regularization in kernel learning},
  author={Mendelson, Shahar and Neeman, Joseph},
  journal={The Annals of Statistics},
  volume={38},
  number={1},
  pages={526--565},
  year={2010},
  publisher={Institute of Mathematical Statistics}
}

@book{steinwart2008support,
  title={Support vector machines},
  author={Steinwart, Ingo and Christmann, Andreas},
  year={2008},
  publisher={Springer Science \& Business Media}
}

@article{kanagawa2018gaussian,
  title={Gaussian processes and kernel methods: A review on connections and equivalences},
  author={Kanagawa, Motonobu and Hennig, Philipp and Sejdinovic, Dino and Sriperumbudur, Bharath K},
  journal={arXiv preprint arXiv:1807.02582},
  year={2018}
}

@inproceedings{jacot2018neural,
  title={Neural tangent kernel: Convergence and generalization in neural networks},
  author={Jacot, Arthur and Gabriel, Franck and Hongler, Clement},
  booktitle={Advances in neural information processing systems},
  pages={8571--8580},
  year={2018}
}

@inproceedings{wasserman2008statistical,
  title={Statistical analysis of semi-supervised regression},
  author={Wasserman, Larry and Lafferty, John D},
  booktitle={Advances in Neural Information Processing Systems},
  pages={801--808},
  year={2008}
}

@inproceedings{yan2016robustness,
  title={On robustness of kernel clustering},
  author={Yan, Bowei and Sarkar, Purnamrita},
  booktitle={Advances in Neural Information Processing Systems},
  pages={3098--3106},
  year={2016}
}

@article{baik2005phase,
  title={Phase transition of the largest eigenvalue for nonnull complex sample covariance matrices},
  author={Baik, Jinho and Arous, G{\'e}rard Ben and P{\'e}ch{\'e}, Sandrine},
  journal={The Annals of Probability},
  volume={33},
  number={5},
  pages={1643--1697},
  year={2005},
  publisher={Institute of Mathematical Statistics}
}

@article{paul2007asymptotics,
  title={Asymptotics of sample eigenstructure for a large dimensional spiked covariance model},
  author={Paul, Debashis},
  journal={Statistica Sinica},
  pages={1617--1642},
  year={2007},
  publisher={JSTOR}
}

@article{pollard1981strong,
  title={Strong consistency of $ k $-means clustering},
  author={Pollard, David},
  journal={The Annals of Statistics},
  volume={9},
  number={1},
  pages={135--140},
  year={1981},
  publisher={Institute of Mathematical Statistics}
}

@article{peng2007approximating,
  title={Approximating k-means-type clustering via semidefinite programming},
  author={Peng, Jiming and Wei, Yu},
  journal={SIAM journal on optimization},
  volume={18},
  number={1},
  pages={186--205},
  year={2007},
  publisher={SIAM}
}

@article{fei2018exponential,
  title={Exponential error rates of SDP for block models: Beyond Grothendieck’s inequality},
  author={Fei, Yingjie and Chen, Yudong},
  journal={IEEE Transactions on Information Theory},
  volume={65},
  number={1},
  pages={551--571},
  year={2018},
  publisher={IEEE}
}

@article{giraud2018partial,
  author    = {Christophe Giraud and
               Nicolas Verzelen},
  title     = {Partial recovery bounds for clustering with the relaxed K-means},
  journal   = {CoRR},
  volume    = {abs/1807.07547},
  year      = {2018},
  url       = {http://arxiv.org/abs/1807.07547},
  archivePrefix = {arXiv},
  eprint    = {1807.07547},
  bibsource = {dblp computer science bibliography, https://dblp.org}
}

@article{gotze2019concentration,
  title={Concentration inequalities for polynomials in $\alpha $-sub-exponential random variables},
  author={G{\"o}tze, Friedrich and Sambale, Holger and Sinulis, Arthur},
  journal={arXiv preprint arXiv:1903.05964},
  year={2019}
}

@article{garey1982complexity,
  title={The complexity of the generalized Lloyd-max problem (corresp.)},
  author={Garey, MR and Johnson, D and Witsenhausen, Hans},
  journal={IEEE Transactions on Information Theory},
  volume={28},
  number={2},
  pages={255--256},
  year={1982},
  publisher={IEEE}
}

@article{aloise2009np,
  title={NP-hardness of Euclidean sum-of-squares clustering},
  author={Aloise, Daniel and Deshpande, Amit and Hansen, Pierre and Popat, Preyas},
  journal={Machine learning},
  volume={75},
  number={2},
  pages={245--248},
  year={2009},
  publisher={Springer}
}

@article{yu2014useful,
  title={A useful variant of the Davis--Kahan theorem for statisticians},
  author={Yu, Yi and Wang, Tengyao and Samworth, Richard J},
  journal={Biometrika},
  volume={102},
  number={2},
  pages={315--323},
  year={2014},
  publisher={Oxford University Press}
}

@article{charikar2002constant,
  title={A constant-factor approximation algorithm for the k-median problem},
  author={Charikar, Moses and Guha, Sudipto and Tardos, {\'E}va and Shmoys, David B},
  journal={Journal of Computer and System Sciences},
  volume={65},
  number={1},
  pages={129--149},
  year={2002},
  publisher={Elsevier}
}

@article{guedon2016community,
  title={Community detection in sparse networks via Grothendieck’s inequality},
  author={Gu{\'e}don, Olivier and Vershynin, Roman},
  journal={Probability Theory and Related Fields},
  volume={165},
  number={3-4},
  pages={1025--1049},
  year={2016},
  publisher={Springer}
}

@book {grothendieck1956resume,
  title = {R{\'e}sum{\'e} of the metrological theory of topological tensor products},
  author = {Grothendieck, Alexander},
  year = {1956},
  publisher = {Soc. of Matem{\'a}tica of S{\~a}o Paulo}
}

@book{vershynin2018high,
  title={High-dimensional probability: An introduction with applications in data science},
  author={Vershynin, Roman},
  volume={47},
  year={2018},
  publisher={Cambridge University Press}
}

@article{birge2001alternative,
  title={An alternative point of view on Lepski’s method},
  author={Birg{\'e}, Lucien},
  journal={Lecture Notes-Monograph Series},
  volume={36},
  pages={113--133},
  year={2001},
  publisher={Institute of Mathematical Statistics}
}

 \newpage
 \appendix

\section{Notation and Preliminaries}

For any $d \in \mathbb{N}$, Let $g:\mybraces{(i,d)}_{i \in [m], d \in [p]} \rightarrow [mp]$ be an injective mapping. For any $(i,d)$, for ease of notation, we simply refer to $g(i,d)$ as $id$ when it occurs as an index. For any $i \in [m]$, $d \in [p]$, $x_{id}$ refers to the $d^{th}$ component of $x_{i}$. For any two tensors $x,y$ $x \otimes y$ refers to the outer product between $x,y$.

\section{Definitions}
\subsection{$\alpha-$ sub-Exponential random variables}
\begin{definition}[\textbf{$\alpha-$ sub-Exponential random variables} \parencite{gotze2019concentration}]
\label{defn:alphasubE}
 A centered random variable $X$ is said to be $\alpha-$ sub-Exponential if there exist two constants $c,C$ and some $\alpha > 0$ such that for all $t \geq 0$,
\[ \Pr \myrounds{\abs{X} \geq t} \leq c \exp \myrounds{-\frac{t^{\alpha}}{C}}\]
The corresponding $\alpha-$ sub-Exponential norm of $X$ is given by:
\[ \mynorm{X}_{\psi_{\alpha}} = \inf \mybraces{t > 0: \E \exp \myrounds{\frac{\abs{X}^{\alpha}}{t^{\alpha}}} \leq 2 }\]
\end{definition}
$\alpha-$ sub-Exponential random variables with $\alpha = 2$ are referred to as sub-Gaussian random variables. Random variables with $\alpha = 1$ sub-Exponential decay are referred to simply as sub-Exponential random variables.

\begin{definition}[\textbf{Tensor norms} \parencite{gotze2019concentration}]
For any $d^{th}$ order, symmetric tensor $A \in \mathbb{R}^{n^d}$, let $\mathcal{J} = \left \{ J_1, J_2,...,J_k \right \}$ be any partition of $[d]$. Then for any $x = x^1 \otimes x^2 \otimes \cdot \cdot \cdot \otimes x^k$, where $x^i \in \mathbb{R}^{n^{\abs{J_i}}}$:
\begin{equation}\label{eqn:TensorNorms}
\norm{A}_{\mathcal{J}} \coloneqq \sup \left \{\sum_{i_1, \ldots, i_d} a_{i_1 \ldots i_d} \prod_{j = 1}^k x_{\mathbf{i}_{J_j}}^{j} : \mynorm{x^j}_2 \le 1 \right \}.
\end{equation}
\end{definition}

\subsubsection{Properties of sub-Gaussian random variables:}
\begin{proposition}[\textbf{Sums of sub-Gaussian random variables} \parencite{vershynin2018high}]
\label{lemma:sumsubG}
Let $\mybraces{X_1,X_2,...,X_m}$ be $m$ independent, centered, sub-Gaussian random variables. Then $\slim_{i \in [m]} X_i$ is a sub-Gaussian random variable and, 
\[ \mynorm{\slim_{i \in [m]} X_i}_{\psitwo}^2 \leq C \slim_{i \in [m]}\mynorm{ X_i}_{\psitwo}^2. \]

\end{proposition}

\begin{proposition}[\textbf{Products of sub-Gaussian random variables }\parencite{vershynin2018high}]
\label{lemma:prdsubG}
Let $X_1$ and $X_2$ be sub-Gaussian random variables. Then $X_1 \cdot X_2$ is a sub-Exponential random variable and,
\[ \mynorm{X_1 \cdot X_2}_{\psione} \leq \mynorm{X_1}_{\psitwo} \cdot \mynorm{X_2}_{\psitwo}. \]
\end{proposition}

\begin{proposition}[\textbf{Squares of sub-Gaussian random variables}\parencite{vershynin2018high}]
\label{lemma:sqrSubG}
Let $X$ be sub-Gaussian random variables. Then $X_1^2$ is a sub-Exponential random variable and,
\[ \mynorm{X_1^2}_{\psione} = \mynorm{X_1}^2_{\psitwo}. \]

\end{proposition}

\section{Useful concentration results}
\begin{proposition}[\textbf{Bernstein's inequality} \parencite{vershynin2018high}]
\label{thm:Bernstein}
 Let $\mybraces{X_1,X_2,...,X_m}$ be a set of independent, centered, sub-Exponential random variables. Then for any $t > 0$, we have: 
 \[ \Pr \myrounds{\abs{\slim_{i \in [m]} X_i} \geq t} \leq 2 \exp \myrounds{-C \min \myrounds{\frac{t^2}{\slim_{i \in [m]} \mynorm{X_i}^2_{\psione}}, \frac{t}{\max \limits_{i \in [m]} \mynorm{X_i}_{\psione} }}}\]
for some fixed constant $C > 0$.
\end{proposition}
\begin{proposition}[\textbf{Tail bounds for chi-squared distributions} \parencite{birge2001alternative}]
\label{thm:chisquared}
The following lower and upper tail bounds hold for non-central chi-squared distributions : For any $t>0$, 
\begin{align}
	\Pr\left(\chi^2_d\left(\mu^2\right) < d + \mu^2 - 2\sqrt{\left(d + 2\mu^2\right)t}\right) < \exp(-t) \label{eq:chi_lower} \\
	\Pr\left(\chi^2_d\left(\mu^2\right) > d + \mu^2 + 2\sqrt{\left(d + 2\mu^2\right)t} + 2t\right) < \exp(-t) \label{eq:chi_upper}
\end{align}

\end{proposition}

\begin{proposition}[\textbf{Polynomials of $\alpha-$ sub-Exponentials}\parencite{gotze2019concentration}]
\label{thm:ConcenPoly}
Let $X_1, \ldots, X_n$ be a set of independent random variables satisfying $\lVert X_i \rVert_{\psi_{2}}\linebreak[3] \le b$ for some $b > 0$. Let $f \colon \mathbb{R}^n \to \mathbb{R}$ be a polynomial of total degree $D \in \mathbb{N}$. Then, for any $t > 0$,
$$\mathbb{P} (|f(X) - \mathbb{E}f(X)| \ge t) \le 2 \exp\Big(- \frac{1}{C_{D}} \min_{1 \le s \le D}\min_{\mathcal{J} \in P_{s}} \Big(\frac{t}{b^s \lVert \mathbb{E} f^{(s)}(X) \rVert_\mathcal{J}} \Big)^{\frac{2}{\abs{\mathcal{J}}}} \Big).$$
where, for any for any $s \in D$, $f^{(s)}$ denotes the symmetric $s^{th}$ order tensor of its $s{th}$ order partial derivatives and $P_s$ denotes the set of all possible partitions of $[s]$.
\end{proposition}

\section{Other useful results}
\begin{proposition}[\textbf{Grothendieck's inequality} \parencite{grothendieck1956resume}]
\label{prop:Grothendieck}For any matrix $A \in \mathbb{R}^{m \times m}$, 
\[ \sup \limits_{\substack{X \succeq 0 \\ diag(X) \leq 1}} \abs{\inner{X}{A}} \leq K_{G} \mynorm{A}_{\infty \rightarrow 1}.\]
where $K_{G} \approx 1.783$ is the Grothendieck's constant.
\end{proposition}

\section{Proofs of lemmas}
 By an application of Bernstein's inequality for sub-exponential random variables, followed by an union bound over all $s,s' \in [k]$, it can be verified that, with high probability,

\begin{equation}
   \min \limits_{s \in [k]} \norm{\mu_s}^2 = p + O(\sqrt{p \log p}); \min \limits_{s \neq s' \in [k]} \inner{\mu_s}{\mu_{s'}} = \frac{-p}{k-1} + O(\sqrt{p \log p}) 
\end{equation}

\begin{proof}[\textbf{Proof of Lemma 9.}]
Let $X = \left \{ x_{id}\right \}_{i \in [m], d \in [p]}$ and let $f(X) = $

\[\slim_{i \neq j \in [m]} y_i z_j (\inner{x_i}{x_j}^2 - \frac{\rho^2}{p^2} \inner{\mu_i}{\mu_j}^2 - p ) + \slim_{i \in [m]} y_i z_i (\mynorm{x_i}^2 - (p + \frac{\rho}{p} \mynorm{\mu_i}^2)^2 - p).  \]

Since $f$ is a $4^{th}$ order polynomial in $X$, from Proposition \ref{thm:ConcenPoly}, for any $t > 0:$

$$\mathbb{P} (|f(X) - \mathbb{E}f(X)| \ge t) \le 2 \exp\Big(- \frac{1}{C_{4}} \min_{1 \le s \le 4}\min_{\mathcal{J} \in P_{s}} \Big(\frac{t}{b^s \lVert \mathbb{E} f^{(s)}(X) \rVert_\mathcal{J}} \Big)^{\frac{2}{\abs{\mathcal{J}}}} \Big).$$

We have $\mathbb{E}f(X) = C_{f}(m^2 \rho + m)$ for some constant $C_f > 0$.

We need the following tensors and their corresponding tensor norms to establish the results of Lemma 9 via an application of Proposition \ref{thm:ConcenPoly}.

\begin{itemize}
    \item For each $s \in [4]$, $s^{th}$ order tensors $A_{s}$ of expectations of $s^{th}$ order derivatives of $f$ with respect to each $ \left \{ x_{i_1,d_1}, ..., x_{i_s,d_s} \right \}_{i_j \in [m], d_j \in [p]}$.
    \item Tensor norms for $A_s$ with respect to each $\mathcal{J}$ in $P_s$ - the set of all possible partitions of [s].
\end{itemize}

\noindent \textbf{Computing $A^1$:}
The first order derivative of $f$ with respect to $x_{id}$ for any $i \in [m]$ and $d \in [p]$: $ \frac{\partial f(X)}{\partial x_{id}} =$
\begin{equation}
    4 x_{id}^3 y_{i} z_{i} + \sum \limits_{d' \neq d} 2 x_{id} x_{id'}^2 y_{i} z_{i} + \sum \limits_{j \neq i} 2 x_{id} x_{jd}^2 y_{i} z_{j} + \sum \limits_{d' \neq d} \sum \limits_{j \neq i} x_{id'} x_{jd} x_{j d'} y_{i} z_{j}.
\end{equation}
\begin{equation}
    \mathbb{E}(\frac{\partial f(X)}{\partial x_{id}}) = O(\sqrt{p \log p}).
\end{equation}
Therefore, $A^1 =  O(\sqrt{p \log p}) \mathbb{J}_{mp}$, where $\mathbb{J}_{mp} \in \mathbb{R}^{mp}$ denotes the vector of ones.

\noindent \textbf{Tensor norms of $A^1$}.
$P_{1} = \left \{1\right \}$ and 
\[\mynorm{A^1}_{\left \{1\right \}} = \sup \left \{ \sum \limits_{1 \in [m], d \in [p]} A^1_{id} x^1_{id} : \mynorm{x^1}_2 \leq 1 \right \} = \mynorm{A^1}_2 =  O(p\sqrt{m\log p}).
\]

All the inequalities in this proof are obtained from multiple applications of H\"{o}lder's inequality with $p =1 $ and $q = \infty$, Cauchy–Schwarz inequality and the inequality: for any $x \in \mathbb{R}^{n}$, $\norm{x}_{1} \leq \sqrt{n} \norm{x}_2$.

\noindent \textbf{Computing $A^2$:}
The second order derivative of $f$ with respect to $x_{id},x_{k \beta}$ for any $i,k \in [m]$ and $d,\beta \in [p]$:

\begin{equation}
 \frac{\partial f(X)}{\partial x_{id} \partial x_{k \beta}} =   \begin{cases}
     12 x_{id}^2 y_i z_i + \sum \limits_{d' \neq d} 2 x_{id'}^2 y_i z_i + \sum \limits_{j \neq i} 2 x_{jd}^2 y_i z_j & \textrm{if } k = i; \beta = d, \\
    4 x_{id} x_{i \beta} y_i z_i + \sum \limits_{j \neq i} 2 x_{jd} x_{j \beta} y_i z_j & \textrm{if } k = i; \beta \neq d, \\
    4 x_{id} x_{kd} y_i z_k + \sum \limits_{d' \neq d} x_{id'} x_{kd'} y_i z_k & \textrm{if } k \neq i; \beta = d, \\
    x_{i \beta} x_{k d} y_i z_k & \textrm{otherwise }.
    \end{cases}
\end{equation}
Then,
\begin{equation}
  A^2(i,j) =  \begin{cases}
     O( p) & \textrm{if } k = i; \beta = d, \\
    O( \log p) & \textrm{if } k = i; \beta \neq d, \\
    O(1) & \textrm{if } k \neq i; \beta = d, \\
    O(\log p / p) & \textrm{otherwise }.
    \end{cases}
\end{equation}

\noindent \textbf{Tensor norms of $A^2$}

$P_2 = \left \{ \left \{ 1,2 \right \}, \left \{  \left \{1  \right \} \left \{ 2 \right \}\right \} \right \}$.

From the definition, its clear that $A^2_{\left \{ 1,2 \right \}} = \mynorm{A}_2 = O(p^2)$.

\[  A^2_{\left \{  \left \{1  \right \} \left \{ 2 \right \}\right \}} = \sup \left \{ \sum \limits_{i,j \in [m], d,d' \in [p]} A^2_{id,jd'} x^1_{id} x^2_{jd'} : \mynorm{x^1}_2, \mynorm{x^2}_2  \leq 1 \right \}  \]
\begin{align*}
&\begin{multlined}
    \leq \sup \limits_{\forall l, \mynorm{x^l}_2 \leq 1} \Big \{ \slim_{i \in [m]} \slim_{d \in [p]} O(p) \abs{x^1_{id} x^2_{id}} + \slim_{i \in [m]} \slim_{d \neq d' \in [p]} O(\log p) \abs{x^1_{id} x^2_{id'}} + \slim_{i \neq j \in [m]} \slim_{d \in [p]} O(1) \abs{x^1_{id} x^2_{jd}}  + \\ \hfill \slim_{i \neq j \in [m]} \slim_{d \neq d' \in [p]} \rho O(\frac{\log p}{p}) \abs{x^1_{id} x^2_{jd'}} \Big \}
\end{multlined}\\
&\begin{multlined}
    \leq \sup \limits_{\forall l, \mynorm{x^l}_2 \leq 1} \Big \{  O(\log p) \slim_{d \neq d' \in [p]} \sqrt{\slim_{i \in [m]} (x^1_{id})^2 \slim_{i \in [m]} (x^2_{id'})^2} + \slim_{i \neq j \in [m]} \slim_{d \in [p]} O(1) \sqrt{\slim_{d \in [p]} (x^1_{id})^2 \slim_{d \in [p]} (x^2_{jd})^2} \\  \hfill + O(p) \mynorm{x^1}_2 \mynorm{x^2}_2 + \rho O(\frac{\log p}{p}) (\sqrt{mp} \mynorm{x^1}_2) (\sqrt{mp} \mynorm{x^2}_2) \Big \}
\end{multlined}\\
&\begin{multlined}
    \leq \sup \limits_{\forall l, \mynorm{x^l}_2 \leq 1} \Big \{ O(p) \mynorm{x^1}_2 \mynorm{x^2}_2 + O(\log p) \sqrt{p} \mynorm{x^1}_2 \sqrt{p} \mynorm{x^2}_2   + O(1) \sqrt{m} \mynorm{x^1}_2 \sqrt{m} \mynorm{x^2}_2 + \\ \hfill \rho O(\frac{\log p}{p}) (\sqrt{mp} \mynorm{x^1}_2) (\sqrt{mp} \mynorm{x^2}_2) \Big \}
\end{multlined}\\
& \leq O(p \log p).
\end{align*}

\noindent \textbf{Computing $A^3$:}
The third order derivative of $f$ with respect to $x_{id},x_{k \beta},x_{\alpha l}$ for any $i,k \in [m]$ and $d,\beta \in [p]$:
\begin{equation*}
  \frac{\partial f(X)}{\partial x_{id} \partial x_{k \beta} \partial x_{\alpha l}} =  \begin{cases}
24 x_{id} y_i z_i & \textrm{if } \alpha = k = i; l = \beta = d, \\
4 x_{il} y_i z_i &\textrm{if } \alpha = k = i; l \neq \beta = d, \\
4 x_{\alpha d} y_i z_{\alpha} &\textrm{if } \alpha \neq k = i; l = \beta = d, \\
x_{\alpha \beta} y_i z_{\alpha} &\textrm{if } \alpha \neq k = i; l = d \neq \beta, \\
0 &\textrm{otherwise }.
\end{cases}
\end{equation*}
Then,
\begin{equation*}
  A^3_{id, k\beta, \alpha l} =  \begin{cases}
O(\frac{\log p}{p}) y_i z_i & \textrm{if } \alpha = k = i; l = \beta = d, \\
O(\frac{\log p}{p}) y_i z_i &\textrm{if } \alpha = k = i; l \neq \beta = d, \\
O(\frac{\log p}{p}) y_i z_{\alpha} &\textrm{if } \alpha \neq k = i; l = \beta = d, \\
O(\frac{\log p}{p}) y_i z_{\alpha} &\textrm{if } \alpha \neq k = i; l = d \neq \beta, \\
0 &\textrm{otherwise }.
\end{cases}
\end{equation*}

\textbf{Tensor norms of $A^3$:}
The set of all possible partitions of $[3]$ up to symmetries is $P_3 = $
\[ \mybraces{ \mybraces{ \mybraces{1,2,3} }, \mybraces{ \mybraces{1},\mybraces{2},\mybraces{3} }, \mybraces{ \mybraces{1}, \mybraces{2,3} }  } 
\]. 

\textbf{Computing $\mynorm{A^3}_{\mybraces{ \mybraces{1,2,3} }}$:}
It follows from the definition that $\mynorm{A^3}_{\mybraces{ \mybraces{1,2,3} }} = \mynorm{A^3}_{2} = O(m \sqrt{p \log p})$.

\textbf{Computing $\mynorm{A^3}_{\mybraces{ \mybraces{1},\mybraces{2},\mybraces{3} }}$:}
\begin{align*}
  &= \sup \mybraces{\sum \limits_{i_1,i_2,i_3} a_{i_1,i_2,i_3} x^1_{i_1} x^2_{i_2} x^2_{i_3} : \forall  l, \mynorm{x^l}_{2} \leq 1} \\
      &\leq \sup \limits_{\forall  l, \mynorm{x^l}_{2} \leq 1} \mybraces{ \slim_{i, j \in [m]} \slim_{d,d' \in [p]} \abs{x^1_{id}} \abs{x^2_{id'}} \abs{x^3_{jd}} } \cdot  O \myrounds{\sqrt{\frac{ \log p}{p}}} \\
    &\leq \sup \limits_{\forall  l, \mynorm{x^l}_{2} \leq 1} \mybraces{ \slim_{i, j \in [m]} \sqrt{ \slim_{d \in [p]} (x^1_{id})^2} \sqrt{ \slim_{d \in [p]} (x^2_{jd})^2} \sqrt{p} \sqrt{ \slim_{d \in [p]} (x^3_{id'})^2} } \cdot  O \myrounds{\sqrt{\frac{ \log p}{p}}} \\
    &\leq \sup \limits_{\forall  l, \mynorm{x^l}_{2} \leq 1} \mybraces{ \sqrt{mp} \mynorm{x^1}_2 \mynorm{x^2}_2 \mynorm{x^3}_2 } \cdot  O \myrounds{\sqrt{\frac{ \log p}{p}}} \\
     & \leq O(\sqrt{p \log p}).
\end{align*}

\textbf{Computing $\mynorm{A^3}_{\mybraces{ \mybraces{1,2},\mybraces{3} }}$:}
\begin{align*}
  &= \sup \mybraces{\sum \limits_{i_1,i_2,i_3} a_{i_1,i_2,i_3} x^1_{i_1} x^2_{i_2,i_3} : \forall  l, \mynorm{x^l}_{2} \leq 1} \\
      &\leq \sup \limits_{\forall  l, \mynorm{x^l}_{2} \leq 1} \mybraces{ \slim_{i, j \in [m]} \slim_{d,d' \in [p]} \abs{x^1_{id}} \abs{x^2_{id',jd}} } \cdot  \mynorm{A^3}_{\infty} \\
    &\leq \sup \limits_{\forall  l, \mynorm{x^l}_{2} \leq 1} \mybraces{ \slim_{i, j \in [m]} \sqrt{ \slim_{d \in [p]} (x^1_{id})^2} \sqrt{p} \sqrt{ \slim_{d' \in [p]} (x^2_{id',jd})^2} } \cdot  O \myrounds{\sqrt{\frac{ \log p}{p}}} \\
    &\leq \sup \limits_{\forall  l, \mynorm{x^l}_{2} \leq 1} \mybraces{ \sqrt{mp} \mynorm{x^1}_2 \mynorm{x^2}_2 } \cdot  O \myrounds{\sqrt{\frac{ \log p}{p}}} \\
     & \leq O(\sqrt{p \log p}).
\end{align*}

\textbf{Computing $A^4$:}
The fourth order derivative of $f$ with respect to $x_{id},x_{k \beta},x_{\alpha l},x_{q \gamma}$ for any $i,k, \alpha,q \in [m]$ and $d,\beta,l, \gamma \in [p]$:
\begin{equation}
  \frac{\partial f(X)}{\partial x_{id} \partial x_{k \beta} \partial x_{\alpha l} \partial x_{q \gamma}} =  \begin{cases}
24 y_i z_i & \textrm{if }q = \alpha = k = i; \gamma = l = \beta = d, \\
4  y_i z_i &\textrm{if } q = \alpha = k = i; \gamma = l \neq \beta = d, \\
4  y_i z_{\alpha} &\textrm{if }q= \alpha \neq k = i; \gamma = l = \beta = d, \\
y_i z_{\alpha} &\textrm{if } q = \alpha \neq k = i; l = d \neq \beta = \gamma, \\
0 &\textrm{otherwise }.
\end{cases}
\end{equation}

\textbf{Tensor norms of $A^4$:}
The list of all possible partitions of $[4]$ is the following(up to symmetries). $P_{[4]} =$

$ \mybraces{ \mybraces{\mybraces{1,2,3,4}}, \mybraces{ \mybraces{1},\mybraces{2}, \mybraces{3}, \mybraces{4}}, \mybraces{ \mybraces{1,2},\mybraces{3,4} }, \mybraces{ \mybraces{1}, \mybraces{2,3,4} }, \mybraces{ \mybraces{1},\mybraces{2}, \mybraces{3,4} }}$.

\textbf{Computation of $ \mynorm{A^4}_{\mybraces{\mybraces{1,2,3,4}}}$:}

Its clear from the definition that  $ \mynorm{A^4}_{\mybraces{\mybraces{1,2,3,4}}} = \mynorm{A}_2 \leq 24mp$.

\textbf{Computation of $ \mynorm{A^4}_{ \mybraces{ \mybraces{1},\mybraces{2}, \mybraces{3}, \mybraces{4}}}$:}
\begin{align*}
  &= \sup \mybraces{\sum \limits_{i_1,i_2,i_3,i_4} a_{i_1,i_2,i_3,i_4} x^1_{i_1} x^2_{i_2} x^2_{i_3} x^4_{i_4} : \forall  l, \mynorm{x^l}_{2} \leq 1} \\
      &=\sup \mybraces{ \slim_{i, j \in [m]} \slim_{d,d' \in [p]} A^{4}_{id,id',jd,jd'} x^1_{id} x^2_{id'} x^3_{jd} x^4_{jd'} : \forall  l, \mynorm{x^l}_{2} \leq 1} \\
      &\begin{multlined}
      \leq \sup \Biggr\{\slim_{i \in [m]} \myrounds{\sqrt{\slim_{d \in [p]} (x^1_{id})^2 \slim_{d' \in [p]}  (x^2_{id'})^2}   }  \slim_{j \in [m]} \myrounds{ \sqrt{\slim_{d \in [p]} (x^3_{jd})^2 \slim_{d' \in [p]} (x^4_{jd'})^2}} : \forall  l, \mynorm{x^l}_{2} \leq 1 \Biggr\} \mynorm{A^4}_{\infty}.
      \end{multlined}\\
     & \leq \mynorm{A^4}_{\infty} = 24.
\end{align*}
Note that the first inequality follows from a simultaneous application of the Holder's inequality with $p= 1$ and $q = \infty$ and the Cauchy schwarz inequality. The last step follows from an application of the Cauchy-Schwarz inequality. 

\textbf{Computation of $A^4_{ \mybraces{ \mybraces{1,2},\mybraces{3,4} }}$ :}
\begin{align*}
  &= \sup \mybraces{\sum \limits_{i_1,i_2,i_3,i_4} a_{i_1,i_2,i_3,i_4} x^1_{i_1,i_2} x^2_{i_3,i_4} : \forall  l, \mynorm{x^l}_{2} \leq 1} \\
  &\begin{multlined}
            =\sup \limits_{\forall  l, \mynorm{x^l}_{2} \leq 1} \Bigg \{ \slim_{i, j \in [m]} \slim_{d,d' \in [p]} A^{4}_{id,id',jd,jd'} \ \left ( x^1_{id,id'} x^2_{jd,jd'} + x^1_{id,jd} x^2_{id',jd'} + x^1_{id,jd'} x^2_{id',jd} + x^1_{id,id'} x^2_{jd',jd} \right.  \\  
             \hfill + x^1_{id,jd} x^2_{jd',id'} + x^1_{id,jd'} x^2_{jd,id'} )  \Bigg \} \\
  \end{multlined}\\[-4\jot]
  & \begin{multlined}
       \leq \sup \limits_{\forall  l, \mynorm{x^l}_{2} \leq 1} \Bigg \{ \slim_{i,j \in [m]} \sqrt{\slim_{d,d' \in [p]} (x^1_{id,id'})^2 \slim_{d,d' \in [p]} (x^2_{jd,jd'})^2} + \slim_{d,d' \in [p]} \sqrt{\slim_{i,j \in [m]} (x^1_{id,jd})^2 \slim_{i,j \in [m]}(x^2_{id',jd'})^2} + \\ \hfill \sqrt{\slim_{i,j \in [m]} \slim_{d,d' \in [p]}(x^1_{id,jd'})^2  \slim_{i,j \in [m]}\slim_{d,d' \in [p]}(x^2_{id',jd})^2 } \Bigg \} 2 \mynorm{A^4}_{\infty} .
  \end{multlined}\\
      &\leq 2 \mynorm{A^4}_{\infty} (m + p + 1) = 48 (m+p+1).
\end{align*}
\textbf{Computation of $A^4_{\mybraces{ \mybraces{1}, \mybraces{2,3,4} }}$: }
\begin{align*}
  &= \sup \mybraces{\sum \limits_{i_1,i_2,i_3,i_4} a_{i_1,i_2,i_3,i_4} x^1_{i_1} x^2_{i_2,i_3,i_4} : \forall  l, \mynorm{x^l}_{2} \leq 1} \\
  &\begin{multlined}
          \leq  \sup \limits_{\forall  l, \mynorm{x^l}_{2} \leq 1} \Bigg \{ \slim_{i, j \in [m]} \slim_{d,d' \in [p]} A^{4}_{id,id',jd,jd'} \ \left ( x^1_{id} x^2_{id',jd,jd'} + x^1_{id} x^2_{id',jd',jd} + x^1_{id} x^2_{id',jd,jd'} + x^1_{id} x^2_{jd',jd,id'} \right. \\ \hfill + x^1_{id} x^2_{jd,jd',id'} + x^1_{id} x^2_{jd',jd,id'} )  \Bigg \} \cdot \mynorm{A^4}_{\infty}, \\
  \end{multlined}\\[-8\jot]
  & \begin{multlined}
        \leq \sup \limits_{\forall  l, \mynorm{x^l}_{2} \leq 1} \Bigg \{ \slim_{i \in [m]} \slim_{d \in [p]} \abs{x^1_{id}} \abs{ \slim_{d' \in [p]} x^2_{id' jd jd'} } \Bigg \} \cdot 6 \mynorm{A^4}_{\infty},
  \end{multlined}\\
  & \begin{multlined}
       \leq \sup \limits_{\forall  l, \mynorm{x^l}_{2} \leq 1} \Bigg \{ \sqrt{p} \slim_{i \in [m]} \sqrt{\slim_{d \in [p]} (x^1_{id})^2} \left (  \slim_{j \in [m]} \sqrt{ \slim_{d,d' \in [p]} (x^2_{id' jd jd'})^2} \right ) \Bigg \} \cdot   6 \mynorm{A^4}_{\infty},
  \end{multlined}\\
  & \begin{multlined}
       \leq \sup \limits_{\forall  l, \mynorm{x^l}_{2} \leq 1} \Bigg \{ \sqrt{mp} \slim_{i \in [m]} \sqrt{\slim_{d \in [p]} (x^1_{id})^2} \left (\sqrt{ \slim_{j \in [m]} \slim_{d,d' \in [p]} (x^2_{id' jd jd'})^2} \right ) \Bigg \} \cdot  6 \mynorm{A^4}_{\infty},
  \end{multlined}\\
    & \begin{multlined}
       \leq \sup \limits_{\forall  l, \mynorm{x^l}_{2} \leq 1} \Bigg \{ \sqrt{mp}  \sqrt{\slim_{i \in [m]} \slim_{d \in [p]} (x^1_{id})^2} \left (\sqrt{ \slim_{i,j \in [m]} \slim_{d,d' \in [p]} (x^2_{id' jd jd'})^2} \right ) \Bigg \} \cdot 6 \mynorm{A^4}_{\infty}
  \end{multlined}\\
      &\leq 6 \mynorm{A^4}_{\infty} (\sqrt{mp}) = 144 \sqrt{mp}.
\end{align*}

\textbf{Computation of $A^4_{\mybraces{ \mybraces{1}, \mybraces{2}, \mybraces{3,4} }}$: }
\begin{align*}
  &= \sup \mybraces{\sum \limits_{i_1,i_2,i_3,i_4} a_{i_1,i_2,i_3,i_4} \cdot x^1_{i_1} ,x^2_{i_2}, x^3_{i_3,i_4} : \forall  l, \mynorm{x^l}_{2} \leq 1}\\
  &\begin{multlined}
            =\sup \limits_{\forall  l, \mynorm{x^l}_{2} \leq 1} \Bigg \{ \slim_{i, j \in [m]} \slim_{d,d' \in [p]} A^{4}_{id,id',jd,jd'} \ \left ( x^1_{id},x^2_{id'} x^3_{jd,jd'} + x^1_{id},x^2_{jd} x^3_{id',jd'} + x^1_{id},x^2_{jd'} x^3_{id',jd} + x^1_{id},x^2_{id'} x^3_{jd',jd} \right.  \\ \hfill  + x^1_{id},x^2_{jd} x^3_{jd',id'} + x^1_{id},x^2_{jd'} x^3_{jd,id'} )  \Bigg \} \\
  \end{multlined}\\[-4\jot]
  & \begin{multlined}
       \leq \sup \limits_{\forall  l, \mynorm{x^l}_{2} \leq 1} \Bigg \{ \slim_{i,j \in [m]} \sqrt{\slim_{d,d' \in [p]} (x^1_{id} x^2_{id'})^2 \slim_{d,d' \in [p]} (x^3_{jd,jd'})^2} + \slim_{d,d' \in [p]} \sqrt{\slim_{i,j \in [m]} (x^1_{id}
        x^2_{jd})^2 \slim_{i,j \in [m]}(x^3_{id',jd'})^2} \\ \hfill + \sqrt{\slim_{i,j \in [m]} \slim_{d,d' \in [p]}(x^1_{id} x^2_{jd'})^2  \slim_{i,j \in [m]}\slim_{d,d' \in [p]}(x^3_{id',jd})^2 } \Bigg \} \cdot 2 \mynorm{A^4}_{\infty} .
  \end{multlined}\\
      &\leq 2 \mynorm{A^4}_{\infty} (m + p + 1) = 48 (m+p+1).
\end{align*}
\end{proof}

Gathering all the norms, we have that for any fixed $y,z \in \mybraces{ \pm 1}^m$: 
\begin{multline}
\label{eqn:fXConcen}
    \mathbb{P} (f(X)  \ge C_{f}(m^2 \rho + m) + t) \le 2 \exp\Bigg(- \frac{1}{C} \min \Bigg(  \Big ( \frac{t}{24mp} \Big )^2, \Big ( \frac{t}{24} \Big )^{\frac{1}{2}}, \Big ( \frac{t}{4(m + p+ 1)} \Big ),    \\ \Big ( \frac{t}{\sqrt{mp}} \Big ), \Big ( \frac{t}{4(m + p+ 1)} \Big )^{\frac{2}{3}}, \Big ( \frac{t}{p \sqrt{p \log p}} \Big )^2, \Big ( \frac{t}{p^2} \Big ), \Big ( \frac{t}{m \sqrt{p \log p}} \Big )^2, \Big ( \frac{t}{ \sqrt{p \log p}} \Big )^{\frac{2}{3}}, \Big ( \frac{t}{\sqrt{mp}} \Big )  \Bigg) \Bigg).
\end{multline}
Applying a union bound over all possible $y,z \in \mybraces{ \pm 1}^m$ and the setting the R.H.S of Equation \ref{eqn:fXConcen} to $\exp(-(1+\epsilon)m \log 2)$, for some arbitrarily small constant $\epsilon >0$ we have that w.h.p, 
\begin{equation*}
    \sup \limits_{\left \{z,y \in \pm 1 \right \}^m} \kappa 
    \sum \limits_{i,j = 1}^{m} y_i z_j R^{(2)}_{i,j}
    \leq \\ \frac{ C_2 \kappa}{p^2}(\rho m(m-1) + m + (mp \sqrt{m} \; \vee \; m^2 \sqrt{m} \; \vee \; p^2 \sqrt{m})).
\end{equation*}
for some constant $C_2 > 0$.
\begin{proof}[\textbf{Proof of Lemma 8}]
For any fixed $y,z \in \mybraces{\pm 1}^{m}$, 
\begin{multline*}
  \slim_{i,j \in [m]} y_i z_j \myrounds{\inner{x_i}{x_j} - \mathbb{E}\inner{x_i}{x_j}} =
     \slim_{d \in [p]} \mysquares{ \myrounds{ \slim_{i \in [m]} y_i x_{id} } \myrounds{\slim_{j \in [m]} z_j x_{jd} } - \E \myrounds{ \slim_{i \in [m]} y_i x_{id} } \myrounds{\slim_{j \in [m]} z_j x_{jd} }  }.
\end{multline*}
Since each $x_{id}$ is a normally distributed random variable, $\slim_{i \in [m]} y_i x_{id}$ is a subGaussian random variable with 
\[ \mynorm{\slim_{i \in [m]} y_i x_{id}}_{\psi_{2}} \leq \sqrt{m}\mynorm{x_{id}}_{\psi_{2}} \leq \sqrt{m}(1 + O (\log p / p)).\]
Therefore, for each $d \in [p]$,  $\myrounds{ \slim_{i \in [m]} y_i x_{id} } \myrounds{\slim_{j \in [m]} z_j x_{jd} }$ is a sub-exponential random variable with sub-exponential norm: 
\[ \mynorm{( \slim_{i \in [m]} y_i x_{id} ) (\slim_{j \in [m]} z_j x_{jd} )}_{\psi_{1}} \leq \mynorm{\slim_{i \in [m]} y_i x_{id}}_{\psi_{2}} \mynorm{\slim_{j \in [m]} z_j x_{jd}}_{\psi_{2}} \leq m(1 + O (\log p / p))^2.\]

Applying Bernstein's inequality for sums of independent sub-exponential random variables, we have that $\forall t > 0$, 
\begin{multline*}
    \Pr \myrounds{ \slim_{d \in [p]} \mysquares{ ( \slim_{i \in [m]} y_i x_{id} ) (\slim_{j \in [m]} z_j x_{jd} ) - \E ( \slim_{i \in [m]} y_i x_{id} ) (\slim_{j \in [m]} z_j x_{jd} ) } > t} \leq \\ \hfill \exp \myrounds{-c \min \myrounds{\frac{t^2}{p m^2 (1 + O(\log p / p))^4}, \frac{t}{m (1 + O(\log p / p))^2}}}.
\end{multline*}
Applying a union bound over all possible partitions $y,z \in \mybraces{\pm 1}^m$, we can see that w.h.p 
\begin{equation*}
     \sup \limits_{y,z \in \mybraces{\pm 1}^m}\slim_{i,j \in [m]} y_i z_j \myrounds{\inner{x_i}{x_j} - \mathbb{E}\inner{x_i}{x_j}} \leq C_1 (\frac{m^2}{\sqrt{\alpha}} \vee m^2)
\end{equation*}
for some constant $C_1 > 0$.
\end{proof}
\begin{proof}[\textbf{Proof of Lemma 1}]
Since for each $i \in [m]$ and each $d \in [p]$, $x_{id}$ is a normally distributed random variable, for each $i,j \in [m]$, $x_{id} x_{jd}$ is a sub-exponential random variable with:
\[ \mynorm{x_{id} x_{jd}}_{\psione} \leq \mynorm{x_{id}}_{\psitwo} \mynorm{x_{jd}}_{\psitwo} \leq (1 + O(\log p / p))^2.\]
From an application of Bernstein's inequality for sub-exponential random variables, we have that:
\[\Pr \myrounds{ \abs{ \inner{x_i}{x_j} - \E \inner{x_i}{x_j}}} > t \leq 2 \exp \myrounds{-c \myrounds{\frac{t^2}{p (1 + O(\log p / p))^4} \wedge \frac{t}{(1 + O(\log p / p))^2}}}.\]
Taking a union bound over all $i,j \in [m]$, we have that: 
\[ \max \limits_{i,j \in [m]} \inner{x_i}{x_j} \leq \E \inner{x_i}{x_j} + O(\sqrt{p \log p}).\]
and 
\[ \min \limits_{i,j \in [m]} \inner{x_i}{x_j} \geq \E \inner{x_i}{x_j} - O(\sqrt{p \log p}).\]
Since $\forall i \neq j$, $\E \inner{x_i}{x_j} = O(1)$, we have that w.h.p:
\[ \max \limits_{i \neq j \in [m]} \frac{\inner{x_i}{x_j}}{p} \leq O(\log p / \sqrt{p}), \; \min \limits_{i \neq j \in [m]} \frac{\inner{x_i}{x_j}}{p} \geq - O(\sqrt{\log p / p}).\]
and $\forall i \in [m]$, $\E \mynorm{x_i}^2 = p + O(1)$. So,
\[\max \limits_{i \in [m]} \frac{\mynorm{x_i}}{p} \leq 1 + O(\log p / \sqrt{p}), \; \min \limits_{i \in [m]} \frac{\mynorm{x_i}}{p} \geq 1 - O(\sqrt{\log p / p}). \]
\end{proof}
\begin{proof} [\textbf{Proof of Lemma 2}]
For any partition $\sigma$ such that $\norm{\beta(\sigma,\sigma_{*})}^2_F   \le 1+(k-1) \epsilon$, 
\[ \frac{k}{m} \sum \limits_{s = 1}^{k} \sum \limits_{\substack{\sigma(i) = s \\ \sigma(j) = s}} \inner{x_i}{x_j} =  \frac{k}{m} \sum \limits_{s = 1}^{k} \norm{\sum \limits_{\substack{\sigma(i) = s}} x_i}^2 .\]
$\sum \limits_{\substack{\sigma(i) = s}} x_i$ is the sum of independent normally distributed random variable and is also normally distributed. Therefore, $\mynorm{\sum \limits_{\substack{\sigma(i) = s}} x_i}^2$ follows a non central chi-square distribution with non-centrality: 
\[ \frac{\alpha \rho}{k} \slim_{s' \in [k]} \slim_{s,t \in [k]} \beta_{s,s'} \beta_{t,s'} \inner{\mu_s}{\mu_t} = \frac{p \alpha \rho}{k-1}(\mynorm{\beta}^2_{F} - 1) + O(\sqrt{p \log p})  \]
and $pk$ degrees of freedom. Applying upper tail bounds from proposition \ref{thm:chisquared}, followed a union bound over all such partitions and setting $t = (1 + \epsilon) m \log k$, we obtain the following inequality which holds with high probability:
\begin{multline*}
	\max_{\substack{\sigma: \norm{\beta(\sigma,\sigma_{*})}^2_F  \\ \le 1+(k-1) \epsilon}} \frac{k}{m} \sum \limits_{s = 1}^{k} \sum \limits_{\substack{\sigma(i) = s \\ \sigma(j) = s}} Q^{1\sigma}_{i,j} \leq k +  \alpha\rho\epsilon +  2(1+\epsilon) \alpha \log k + 2 \sqrt{ (1+\epsilon) \left(k + 2 \alpha\rho \epsilon \right)\alpha \log k }
	 + O(\sqrt{\log p/p} ).
\end{multline*}
Similarly, the random variable $\sum \limits_{i=1}^{m} \sum \limits_{d=1}^{p}x_{id}^2$ is distributed according to a non-central chi-squared distribution with non-centrality$(\mu^2)$ $p \alpha \rho $ and $mp$ degrees of freedom$(d)$. Note that it is independent of the partition. Using the lower tail bounds from proposition \ref{thm:chisquared} and setting $t = \log(p)$, w.p.a.l $(1 - \frac{1}{p})$.
\begin{equation*}
    \max_{\substack{\sigma: \norm{\beta(\sigma,\sigma_{*})}^2_F  \\ \le 1+(k-1) \epsilon}} - \gamma_{\max}Q^{5}_{i} \leq -\frac{k\gamma_{\max}\tau}{mp} \left ( mp + p \alpha \rho - 2 \sqrt{\left ( mp + 2p\alpha\rho \right ) \log p} \right ).
\end{equation*}
\end{proof}
\begin{proof}[\textbf{Proof of Lemma 4}]
Using the inequality, $\sum \limits_{i = 1}^{n} a_{i} \cdot b_{i} \leq \sup \limits_{i \in [n]} \abs{b_{i}} \cdot \sum \limits_{i = 1}^{n} \abs{a_{i}}$, we have:
\[\frac{k}{m} \slim_{i \in [m]} (\frac{\norm{x_i}^2}{p} - \tau )^2  \leq k  \max \limits_{i \in [m]} (\frac{\norm{x_i}^2}{p} - \tau )^2 \leq k O(\frac{\log p}{p}). \]
Therefore, 
\begin{equation}
     \max_{\substack{\sigma: \norm{\beta(\sigma,\sigma_{*})}^2_F  \\ \le 1+(k-1) \epsilon}} \slim_{i \in [m]} \frac{k \gamMax (e^{\tau} - 1)}{2m} Q_{i}^{4 \sigma} \leq C_{0} k \gamMax (e^{\tau}-1) (\log p)^2/2p.
\end{equation}
\end{proof}
\begin{proof}[\textbf{Proof of Lemma 5}]
For the true partition $\sigma^*$, $\mynorm{\sum \limits_{\substack{\sigma^*(i) = s}} x_i}^2$ follows a non central chi-square distribution with non-centrality: $ p \alpha \rho + O(\sqrt{p \log p})$ and $pk$ degrees of freedom. Applying lower tail bounds from proposition \ref{thm:chisquared} and setting $t = \log p$, we obtain the following inequality which holds with high probability:
\begin{equation}
    \frac{k}{m} \sum \limits_{s = 1}^{k} \sum \limits_{\substack{\sigma^*(i) = s \\ \sigma^*(j) = s}} Q^{1\sigma}_{i,j} \geq  k + \alpha \rho - O(\sqrt{\log p/p}).
\end{equation}
Similarly, as noted earlier, the random variable $\sum \limits_{i=1}^{m} \sum \limits_{d=1}^{p}x_{id}^2$ is distributed according to a non-central chi-squared distribution with non-centrality$(\mu^2)$ $p \alpha \rho $ and $mp$ degrees of freedom$(d)$. Using the upper tail bounds from proposition \ref{thm:chisquared} and setting $t = \log(p)$, w.p.a.l $(1 - \frac{1}{p})$:
\begin{equation}
      - \gamma_{\min} Q^{5} >  -\frac{k\gamma_{\min}\tau}{mp} \left ( mp + p \alpha \rho +2\log p 
    - 2 \sqrt{\left ( mp + 2p\alpha\rho \right )\log p} \right ).
\end{equation}
\end{proof}
\begin{proof}[\textbf{Proof of Lemma 7}]
\[\Ktilde(i,j) = f(0) + \begin{cases} \frac{f^{'}(0) \rho \inner{\mu_{i}}{\mu_{j}}}{p^2} + \frac{ \kappa \rho^2 \inner{\mu_i}{\mu_j}^2}{p^4} + \frac{\kappa}{p} & \textrm{if } i \neq j \\
\frac{f^{'}(0)(p^2 + \rho \norm{\mu_{i}}^2)}{p^2} + \frac{ \kappa (p^2 + \rho \norm{\mu_{i}}^2)^2}{p^4} + \frac{\kappa}{p}  &\textrm{otherwise}.
\end{cases}\]
\begin{align*}
    \inner{\Ktilde}{X^* - \hat{X}} &= \slim_{S} \slim_{i \neq j \in S} \Ktilde_{i,j} (1 - \hat{X}_{i,j}) + \slim_{S,S'} \slim_{ \substack{i \in S \\j \in S'}} \Ktilde_{i,j} ( - \hat{X}_{i,j}) \\
    & \geq \min \limits_{S} \min \limits_{i \neq j \in S} \Ktilde_{i,j}\slim_{S} \slim_{i \neq j \in S} (1 - \hat{X}_{i,j})  - \max \limits_{S,S'} \max \limits_{ \substack{i \in S \\j \in S'}} \Ktilde_{i,j} \slim_{ \substack{i \in S \\j \in S'}} \slim_{ \substack{i \in S \\j \in S'}} (\hat{X}_{i,j}) \\
    &= (\min \limits_{S} \min \limits_{i \neq j \in S} \Ktilde_{i,j} - \max \limits_{S,S'} \max \limits_{ \substack{i \in S \\j \in S'}} \Ktilde_{i,j} )\slim_{S} \slim_{i \neq j \in S} (1 - \hat{X}_{i,j}).
\end{align*}
The last inequality is obtained using the property that the sum of entries of each row of $\hat{X}$ is equal to $\frac{m}{k}$.
Similarly,
\begin{align*}
    \mynorm{X^* - \hat{X}}_1 &= \slim_{S} \slim_{i \neq j \in S} (1 - \hat{X}_{i,j}) + \slim_{S,S'} \slim_{ \substack{i \in S \\j \in S'}}  ( \hat{X}_{i,j}) \leq 2 \slim_{S} \slim_{i \neq j \in S} (1 - \hat{X}_{i,j}).
\end{align*}
Therefore, 
\begin{align}
    \mynorm{X^* - \hat{X}}_1 \leq \frac{2}{(\min \limits_{S} \min \limits_{i \neq j \in S} \Ktilde_{i,j} - \max \limits_{S,S'} \max \limits_{ \substack{i \in S \\j \in S'}} \Ktilde_{i,j} )}  \inner{\Ktilde}{X^* - \hat{X}}.
\end{align}
Substituting the values of $\Ktilde_{i,j}$, we have that %
\begin{align*}
    \min \limits_{S} \min \limits_{i \neq j \in S} \Ktilde_{i,j} &= f(0) + f'(0)  \min \limits_{S} \frac{\rho\mynorm{\mu_s}^2}{p^2} + e^{\tau} \gamMax  \min \limits_{S} \frac{\rho^2 \mynorm{\mu_s}^4}{p^4} \\
    &= f(0) + f'(0) \frac{ \rho (1 + O(\sqrt{\frac{\log p}{p}}))}{p} + e^{\tau} \gamMax \frac{ \rho^2 (1 + O(\sqrt{\frac{\log p}{p}}))^2}{p^2}.
\end{align*}
\begin{align*}
    \max \limits_{S \neq S'} \max \limits_{i \in S, j \in S'} \Ktilde_{i,j} &= f(0) + f'(0) \max \limits_{S \neq S'} \frac{\rho \inner{\mu_s}{\mu_s'}}{p^2} + e^{\tau} \gamMax \max \limits_{S \neq S'} \frac{\rho^2 \inner{\mu_s}{\mu_s'}^2}{p^4} \\
    &= f(0) + f'(0) \frac{ \rho (\frac{-1}{k-1} + O(\sqrt{\frac{\log p}{p}}))}{p} + e^{\tau} \gamMax \frac{ \rho^2 (\frac{-1}{k-1}+ O(\sqrt{\frac{\log p}{p}}))^2}{p^2}.
\end{align*}
where, the second equalities for both quantities arise from substituting the values of $\min \limits_{s} \mynorm{\mu_s}^2$ and $\min \limits_{s,s'} \inner{\mu_s}{\mu_{s'}}$. Therefore, 
\begin{align*}
  \min \limits_{S} \min \limits_{i \neq j \in S} \Ktilde_{i,j} -  \max \limits_{S \neq S'} \max \limits_{i \in S, j \in S'} \Ktilde_{i,j} =  \frac{\rho}{p}\myrounds{\frac{k}{k-1} + O(\sqrt{\frac{\log p}{p}}) + O(1/p)} \textrm{ and }
\end{align*} 
\begin{equation*}
     \mynorm{X^* - \hat{X}}_1 \leq \frac{2 \inner{\Ktilde}{X^* - \hat{X}}}{\frac{\rho}{p}\myrounds{\frac{k}{k-1} + O(\sqrt{\frac{\log p}{p}}) + O(1/p)}}  .
\end{equation*}
\end{proof}
\begin{proof}[\textbf{Proofs of Lemma 3,6}]
$f_{\sigma_{*}}(X) = \slim_{s \in [k]} \slim_{i,j \in \sigma_*^{-1}(s)} \inner{x_i}{x_j}^2$.
\begin{align*}
    \E \slim_{s \in [k]} \slim_{i,j \in \sigma_*^{-1}(s)} \inner{x_i}{x_j}^2 = \frac{mp^2}{k}(k + \frac{\alpha}{k} + O(\frac{1}{p})).
\end{align*}
We need the following tensors and their corresponding tensor norms to establish the results of lemma 3 and 6 via an application of Proposition \ref{thm:ConcenPoly}.
\begin{itemize}
    \item For each $s \in [4]$, $s^{th}$ order tensors $A_{s}$ of expectations of $s^{th}$ order derivatives of $f$ with respect to each $ \left \{ x_{i_1,d_1}, ..., x_{i_s,d_s} \right \}_{i_j \in [m], d_j \in [p]}$.
    \item Tensor norms for $A_s$ with respect to each $\mathcal{J}$ in $P_s$.
\end{itemize}
\noindent \textbf{Computing $A^1$:}
The first order derivative of $f$ with respect to $x_{id}$ for any $i \in [m]$ and $d \in [p]$: $ \frac{\partial f_{\sigma_{*}}(X)}{\partial x_{id}} =$
\begin{equation*}
    4 x_{id}^3  + \sum \limits_{d' \neq d} 2 x_{id} x_{id'}^2  + \slim_{s \in [k]} \sum \limits_{j \neq i \in \sigma_*^{-1}(s)} 2 x_{id} x_{jd}^2  + \sum \limits_{d' \neq d} \sum \limits_{j \neq i \in \sigma_*^{-1}(s)} x_{id'} x_{jd} x_{j d'}.
\end{equation*}
\begin{equation*}
    \mathbb{E}(\frac{\partial f_{\sigma_{*}}(X)}{\partial x_{id}}) = O(\sqrt{p \log p}).
\end{equation*}
Therefore, \[ A^1 =  O(\sqrt{p \log p}) \mathbb{J}_{mp} \], where $\mathbb{J}_{mp} \in \mathbb{R}^{mp}$ denotes the vector of ones.

\noindent \textbf{Computing $A^2$:}
The second order derivative of $f$ with respect to $x_{id},x_{k \beta}$ for any $i,k \in [m]$ and $d,\beta \in [p]$:
\begin{equation*}
  \frac{\partial f_{\sigma_{*}}(X)}{\partial x_{id} \partial x_{k \beta}} =  \begin{cases}
     12 x_{id}^2 + \sum \limits_{d' \neq d} 2 x_{id'}^2 + \sum \limits_{j \neq i \in \sigma_*^{-1}(s)} 2 x_{jd}^2  & \textrm{if } k = i; \beta = d, \\
    4 x_{id} x_{i \beta} + \sum \limits_{j \neq i \in \sigma_*^{-1}(s) } 2 x_{jd} x_{j \beta}& \textrm{if } k = i; \beta \neq d, \\
    4 x_{id} x_{kd} + \sum \limits_{d' \neq d} x_{id'} x_{kd'} & \textrm{if } k \neq i \in \sigma_*^{-1}(s); s \in [k]; \beta = d, \\
    x_{i \beta} x_{k d} & \textrm{if } k \neq i \in \sigma_*^{-1}(s); s \in [k]; \beta \neq d, \\
    0 &\textrm{otherwise }.
    \end{cases}
\end{equation*}
Then,
\begin{equation*}
   A^2(i,j) = \begin{cases}
     O(p) & \textrm{if } k = i; \beta = d, \\
    O( \log p) & \textrm{if } k = i; \beta \neq d, \\
    O(1) & \textrm{if } k \neq i \in \sigma_*^{-1}(s); s \in [k]; \beta = d, \\
    O(\log p / p) & \textrm{if } k \neq i \in \sigma_*^{-1}(s); s \in [k]; \beta \neq d, \\
    0 &\textrm{otherwise }.
    \end{cases}
\end{equation*}

\noindent \textbf{Computing $A^3$:}
The third order derivative of $f_{\sigma_{*}}(X)$ with respect to $x_{id},x_{k \beta},x_{\alpha l}$ for any $i,k \in [m]$ and $d,\beta \in [p]$:
\begin{equation*}
  \frac{\partial f_{\sigma_{*}}(X)}{\partial x_{id} \partial x_{k \beta} \partial x_{\alpha l}} =  \begin{cases}
24 x_{id} & \textrm{if } \alpha = k = i; l = \beta = d, \\
4 x_{il} &\textrm{if } \alpha = k = i; l \neq \beta = d, \\
4 x_{\alpha d} &\textrm{if } \alpha \neq k = i; l = \beta = d, \alpha \neq i \in \sigma_*^{-1}(s); s \in [k];\\
x_{\alpha \beta} &\textrm{if } \alpha \neq k = i; l = d \neq \beta, \alpha \neq i \in \sigma_*^{-1}(s); s \in [k]; \\
0 &\textrm{otherwise }.
\end{cases}
\end{equation*}
Then,
\begin{equation*}
  A^3_{id, k\beta, \alpha l} =  \begin{cases}
O(\frac{\log p}{p}) & \textrm{if } \alpha = k = i; l = \beta = d, \\
O(\frac{\log p}{p}) &\textrm{if } \alpha = k = i; l \neq \beta = d, \\
O(\frac{\log p}{p}) &\textrm{if } \alpha \neq k = i; l = \beta = d, \\
O(\frac{\log p}{p}) &\textrm{if } \alpha \neq k = i; l = d \neq \beta, \\
0 &\textrm{otherwise }.
\end{cases}
\end{equation*}

\textbf{Computing $A^4$:}
The fourth order derivative of $f_{\sigma_{*}}(X)$ with respect to $x_{id},x_{k \beta},x_{\alpha l},x_{q \gamma}$ for any $i,k, \alpha,q \in [m]$ and $d,\beta,l, \gamma \in [p]$:
\begin{equation*}
  \frac{\partial f_{\sigma_{*}}(X)}{\partial x_{id} \partial x_{k \beta} \partial x_{\alpha l} \partial x_{q \gamma}} =  \begin{cases}
24 & \textrm{if }q = \alpha = k = i; \gamma = l = \beta = d, \\
4  &\textrm{if } q = \alpha = k = i; \gamma = l \neq \beta = d, \\
4 &\textrm{if }q= \alpha \neq k = i; \gamma = l = \beta = d, \alpha \neq i \in \sigma_*^{-1}(s); s \in [k]; \\
1 &\textrm{if } q = \alpha \neq k = i; l = d \neq \beta = \gamma, \alpha \neq i \in \sigma_*^{-1}(s); s \in [k]; \\
0 &\textrm{otherwise }.
\end{cases}
\end{equation*}

Computing all the tensor norms, (see the proof of Lemma 9 for how the norms are computed) we have:
\[ \begin{matrix*}[l]%
\mynorm{A^1}_{\mybraces{1}} = O(p \sqrt{p \log p}); & \mynorm{A^2}_{\mybraces{1,2}} = O(p^2) ; & \mynorm{A^2}_{\mybraces{\mybraces{1},\mybraces{2}}} = O(p \log p) ;\\
 \mynorm{A^3}_{\mybraces{1,2,3}} = O(m \sqrt{p \log p}); & \mynorm{A^3}_{\mybraces{1,2},\mybraces{3}} = O(\sqrt{mp}) ; & \mynorm{A^3}_{\mybraces{\mybraces{1},\mybraces{2},\mybraces{3}}} = O(\sqrt{p \log p}) ;\\
\mynorm{A^4}_{\mybraces{1,2,3,4}} = O(mp); & \mynorm{A^4}_{\mybraces{1,2},\mybraces{3,4}} = O(p) ; & \mynorm{A^4}_{\mybraces{\mybraces{1},\mybraces{2},\mybraces{3},\mybraces{4}}} = O(1) ; \\
 \mynorm{A^4}_{\mybraces{1},\mybraces{2,3,4}} = O(\sqrt{mp}); & \mynorm{A^4}_{\mybraces{\mybraces{1},\mybraces{2},\mybraces{3,4}}} = O(p) ; &
\end{matrix*}\]%

Applying the lower tail bounds for $f_{\sigma_{*}}(X)$ from Proposition \ref{thm:ConcenPoly}, and setting the R.H.S of the inequality to $\frac{1}{p}$, we derive the following upper bound that holds with probability at least $1 - 1/p$:
\begin{equation*}
   f_{\sigma_{*}}(X) > \frac{mp^2}{k}(k + \frac{\alpha}{k} + O(\frac{1}{p})) - ( O(p^2 \sqrt{\log p}) \vee O(mp \sqrt{\log p})) 
\end{equation*}
Therefore,
\begin{equation*}
      \gamma_{\min} Q_{2 \sigma_{*}} >  \gamma_{\min} \left (1 + \frac{1}{k} + O(\frac{1}{p}) \right ) - C_{2} \gamma_{\min} \left ( \sqrt{\frac{\log p}{p^2}} \vee  \alpha \sqrt{\frac{\log p}{p^2}} \right ).
\end{equation*}
Fix some $\epsilon > 0$ be an arbitrarily small constant. For any $ \sigma: \mynorm{\beta(\sigma, \sigma_{*})}^2_{F} < 1 + (k-1) \epsilon$, let \[f_{\sigma}(X) = \slim_{s \in [k]} \slim_{i,j \in \sigma^{-1}(s)} \inner{x_i}{x_j}^2.\]
We can show that $\E f_{\sigma}(X) \leq \frac{mp^2}{k}(k + \frac{\alpha}{k} + O(\frac{1}{p}))$. Computing the tensors and their respective norms similarly as above, applying proposition \ref{thm:ConcenPoly}, followed by an union bound over all such partitions and we can show that w.h.p, 

\begin{equation*}
  \max_{\substack{\sigma: \norm{\beta(\sigma,\sigma_{*})}^2_F  \\ \le 1+(k-1) \epsilon}} \gamma_{\max}  Q_{2\sigma} \leq \gamma_{\max} \left (1 + \frac{1}{k} + O(\frac{1}{p}) \right )  + C_{2} \gamma_{\max} O \left ( \sqrt{\frac{\alpha}{p}} \vee  \alpha \sqrt{\frac{\alpha}{p}} \vee \sqrt{\frac{1}{\alpha p}} \right ).
\end{equation*}

\end{proof}

\end{document}